\documentclass[runningheads]{llncs}

\usepackage{eccv}

\usepackage{eccvabbrv}

\usepackage{graphicx}
\usepackage{booktabs}
\usepackage{pifont}

\usepackage{colortbl}
\usepackage{color,xcolor}
\usepackage{algorithm,algorithmic}

\usepackage[accsupp]{axessibility}  

\usepackage{hyperref}

\usepackage{orcidlink}

\begin{document}

\title{Vision-TTT: Efficient and Expressive Visual Representation Learning with Test-Time Training} 

\titlerunning{Vision-TTT}

\author{Quan Kong\inst{1}\orcidlink{} \and
Yanru Xiao\inst{2}\orcidlink{} \and
Yuhao Shen\inst{1}\orcidlink{} \and
Cong Wang\inst{1}\orcidlink{}}

\authorrunning{Q.~Kong et al.}

\institute{${}^\mathsection$Zhejiang University, Hangzhou, China \and
${}^\diamondsuit$Amazon Web Services, USA\\
\email{\{qkv, riven, cwang85\}@zju.edu.cn}\\
\email{yxiao002@odu.edu} 
}

\maketitle

\begin{abstract}
Learning efficient and expressive visual representation has long been the pursuit of computer vision research. While Vision Transformers (ViTs) gradually replace traditional Convolutional Neural Networks (CNNs) as more scalable vision learners, their applications are plagued by the quadratic complexity of the self-attention mechanism. To address the challenge, we introduce a new linear-time sequence modeling method Test-Time Training (TTT) into vision and propose Vision-TTT, which treats visual sequences as datasets and compresses the visual token sequences in a novel self-supervised learning manner. By incorporating the dual-dataset strategy and Conv2d-based dataset preprocessing, Vision-TTT effectively extends vanilla TTT to model 2D visual correlations with global receptive fields. Extensive experiments show that \texttt{Vittt-T/S/B} achieve $77.7\%,81.8\%,82.7\%$ Top-1 accuracy on ImageNet classification and also greatly outperform their counterparts on downstream tasks. At $1280\times1280$ resolution, \texttt{Vittt-T} reduces FLOPs by $79.4\%$ and runs $4.72\times$ faster with $88.9\%$ less memory than DeiT-T. These results demonstrate the expressiveness and efficiency of Vision-TTT as a strong candidate for the next-generation generic visual backbone. 
\vspace{-0.1in}
  \keywords{Visual representation learning \and Linear Complexity Visual Sequence Modeling \and Test-Time Training }
\end{abstract}

\section{Introduction}
\label{sec:intro}
Learning efficient and expressive visual representation learning has long been a fundamental task in computer vision research. Convolutional Neural Networks (CNNs)~\cite{VGG,ResNet,DenseNet,ResNeXt,HRNet,EfficientNetV1,EfficientNetV2} represent classic architectures in this domain. They can efficiently capture spatial hierarchies, yet are limited by the static nature of convolutional kernels, which restricts their performance scalability. In recent years, Vision Transformers (ViTs)~\cite{dosovitskiy2020vit,deit,CrossAttention, Container,t2tViT,SwinTransformer} have explored the self-attention mechanism for visual sequence modeling. With superior scalability demonstrated on large‑scale experiments~\cite{CLIP,Simeoni2025DINOv3}, ViTs have gradually become a standard architectural paradigm in both academia and industry.

However, as the community’s demand for high-resolution image understanding continues to grow, ViTs are severely plagued by their quadratic computational complexity~\cite{Transformer} when processing long-sequence tasks. To tackle this bottleneck, researchers have been striving to pursue a Pareto-optimal architecture that balances expressiveness and efficiency. Among the emerging visual sequence models~\cite{vim, VMamba, MambaVisionAH, VisionRWKVEA, ViG}, Vision-Mambas~\cite{vim, VMamba} represent the most popular and representative approaches that build upon the State Space Models (SSMs) with selective scan design and hardware-aware parallelism~\cite{mamba, mamba2}.

In this paper, we build on a new linear-time sequence modeling paradigm Test-Time Training (TTT)~\cite{sun2024learning} and propose the Vision-TTT architecture, which establishes a new Pareto front between the expressiveness and efficiency in visual representation. In the context of visual sequence modeling~\cite{dosovitskiy2020vit,vim,VisionRWKVEA,ViG}, the core idea of TTT is to treat the visual token sequence of an image as a dataset, along which self-supervised learning is performed to compress visual semantics into a hidden state via gradient updates. As a result, the token semantics are explicitly guided by the gradients to form an expressive and explainable visual representation. Furthermore, linear-complexity efficiency can also be achieved through hardware-aware matrix multiplication parallelism~\cite{sun2024learning}. 

\begin{figure*}[t]
    \centering 
    \begin{subfigure}[b]{0.32\linewidth}
        \centering
        \includegraphics[width=1.0\linewidth]{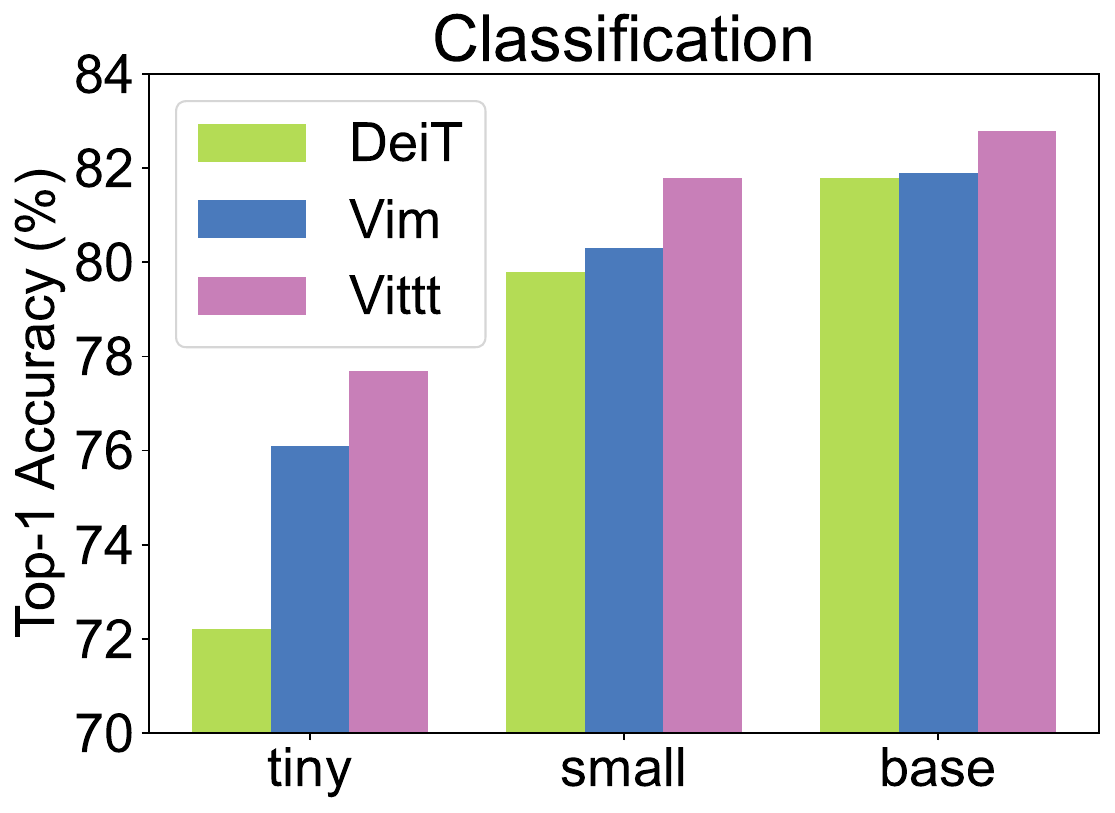} 
    \end{subfigure}
    \hfill
    \begin{subfigure}[b]{0.32\linewidth}
        \centering
        \includegraphics[width=1.0\linewidth]{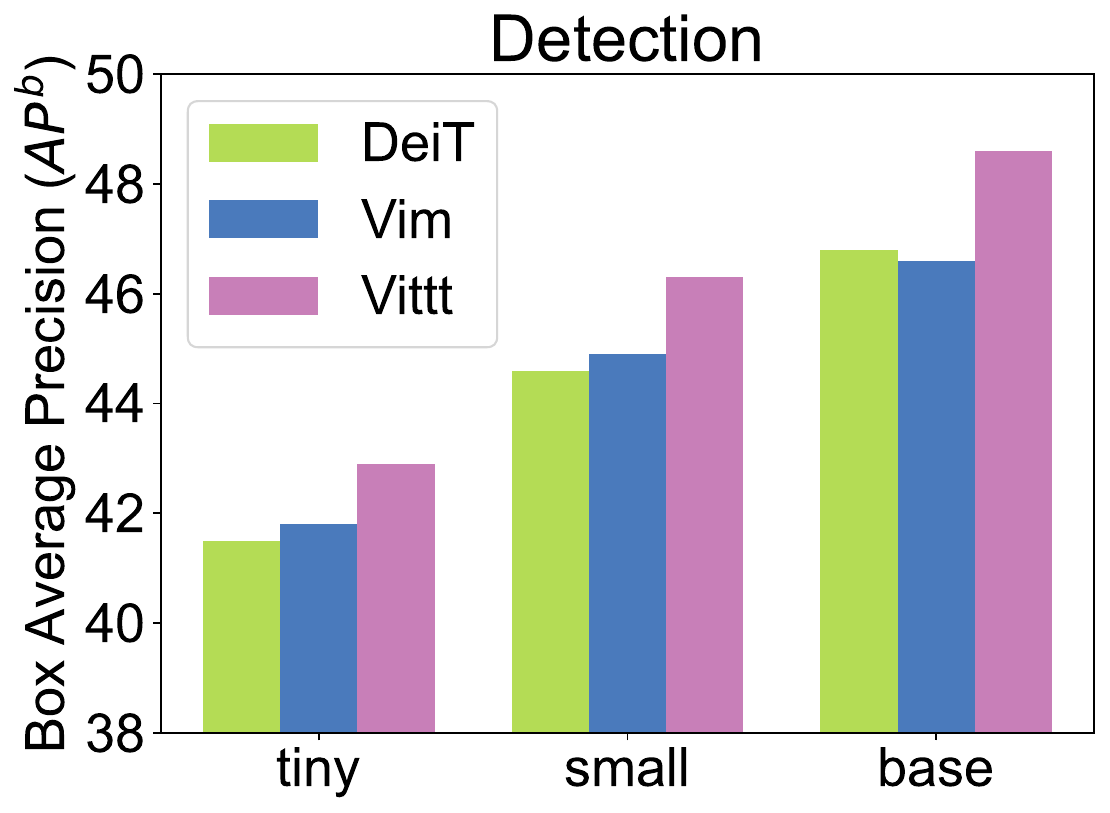} 
    \end{subfigure}
    \hfill
    \begin{subfigure}[b]{0.32\linewidth}
        \centering
        \includegraphics[width=1.0\linewidth]{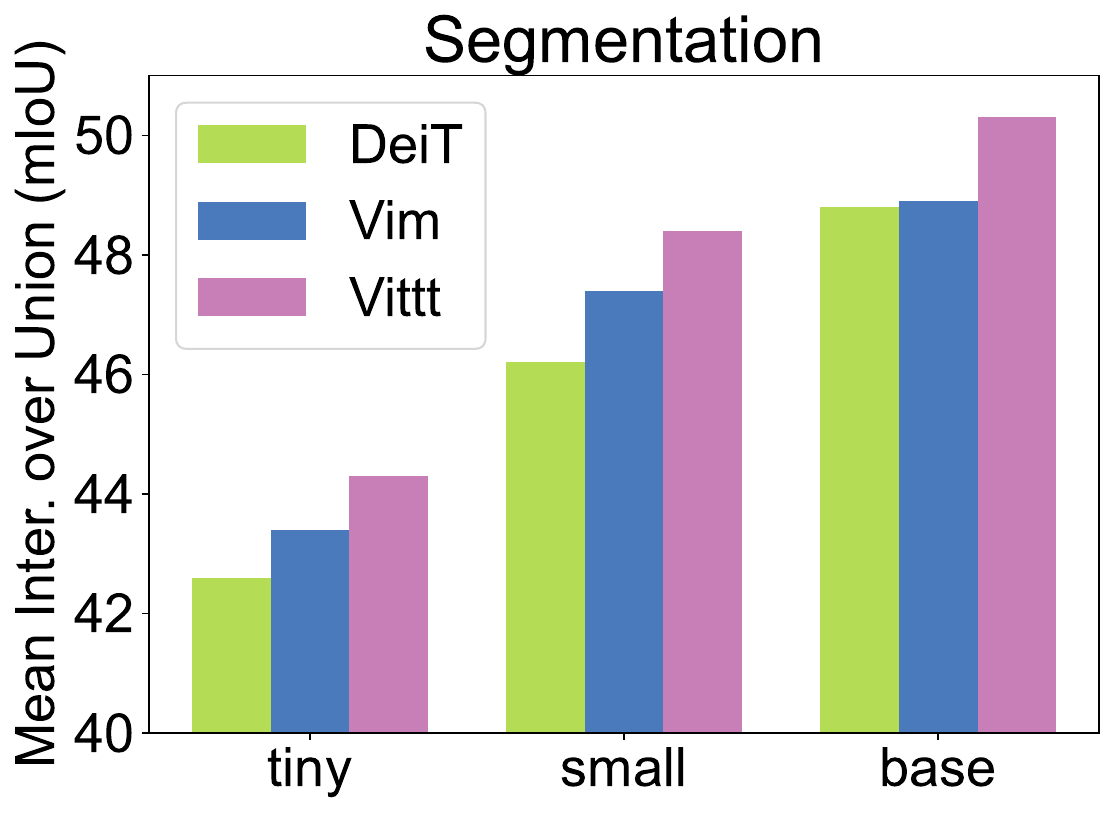} 
    \end{subfigure}
    
    \begin{subfigure}[b]{0.32\linewidth}
        \centering
        \includegraphics[width=1.0\linewidth]{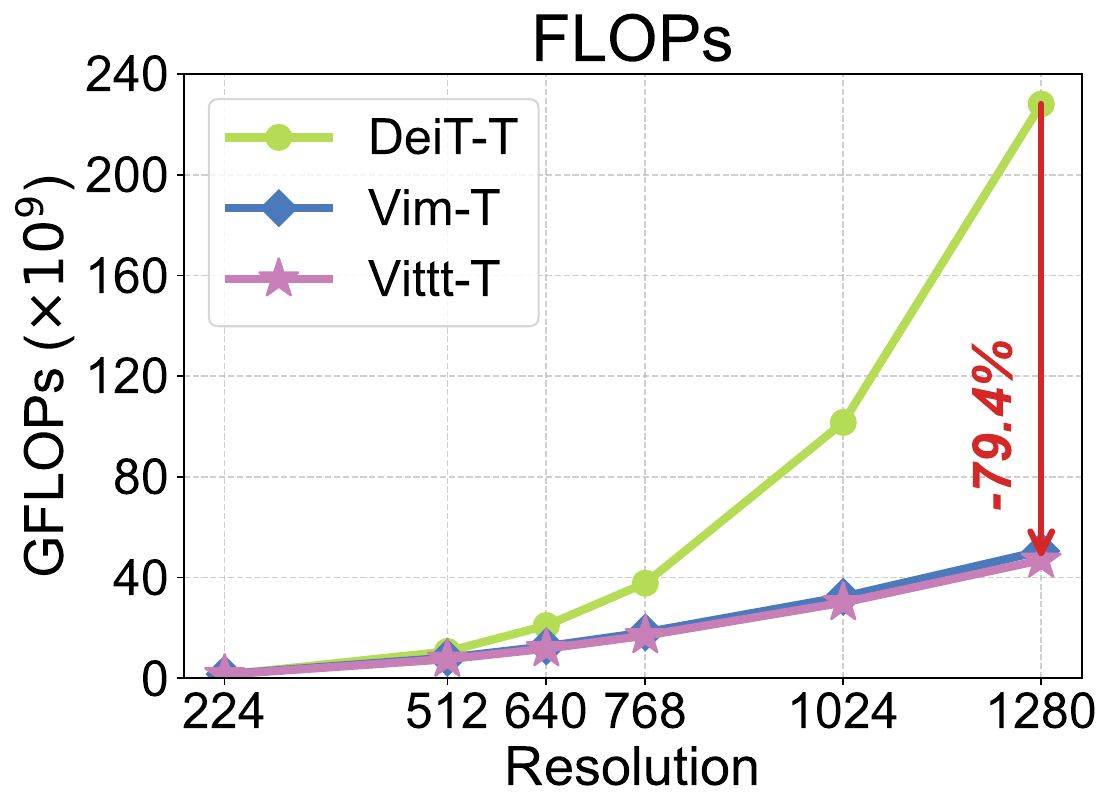} 
    \end{subfigure}
    \hfill
    \begin{subfigure}[b]{0.32\linewidth}
        \centering
        \includegraphics[width=1.0\linewidth]{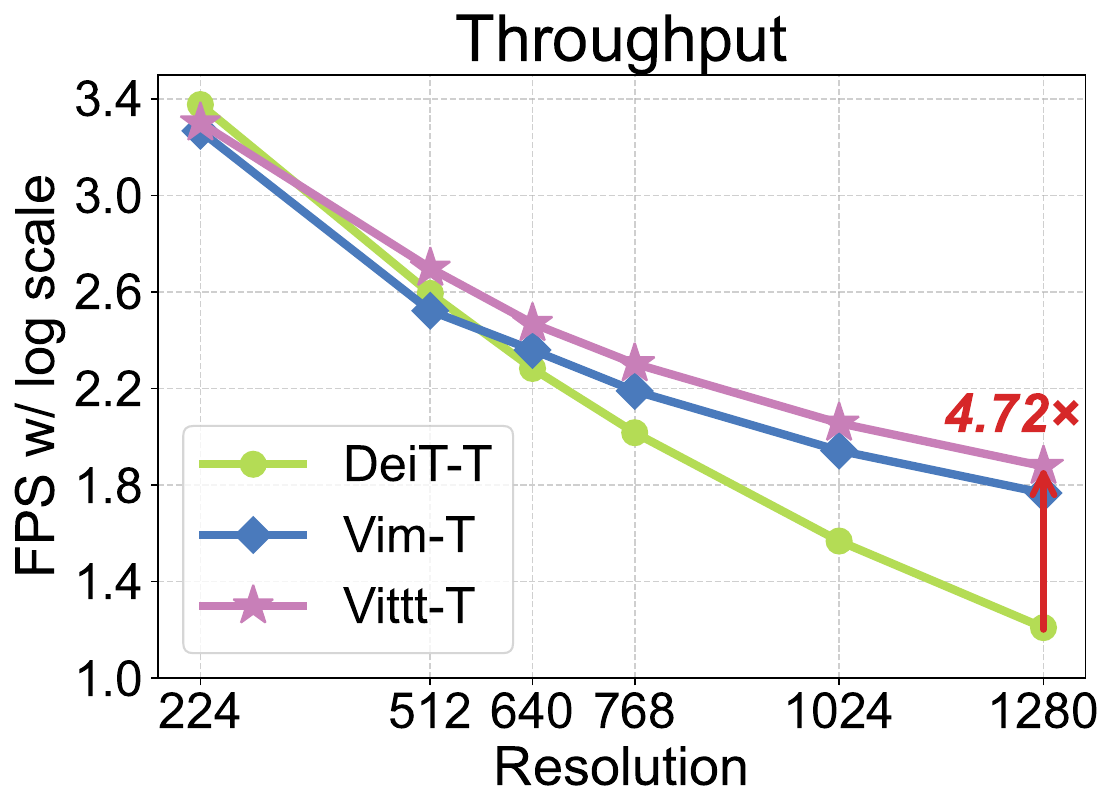} 
    \end{subfigure}
    \hfill
    \begin{subfigure}[b]{0.32\linewidth}
        \centering
        \includegraphics[width=1.0\linewidth]{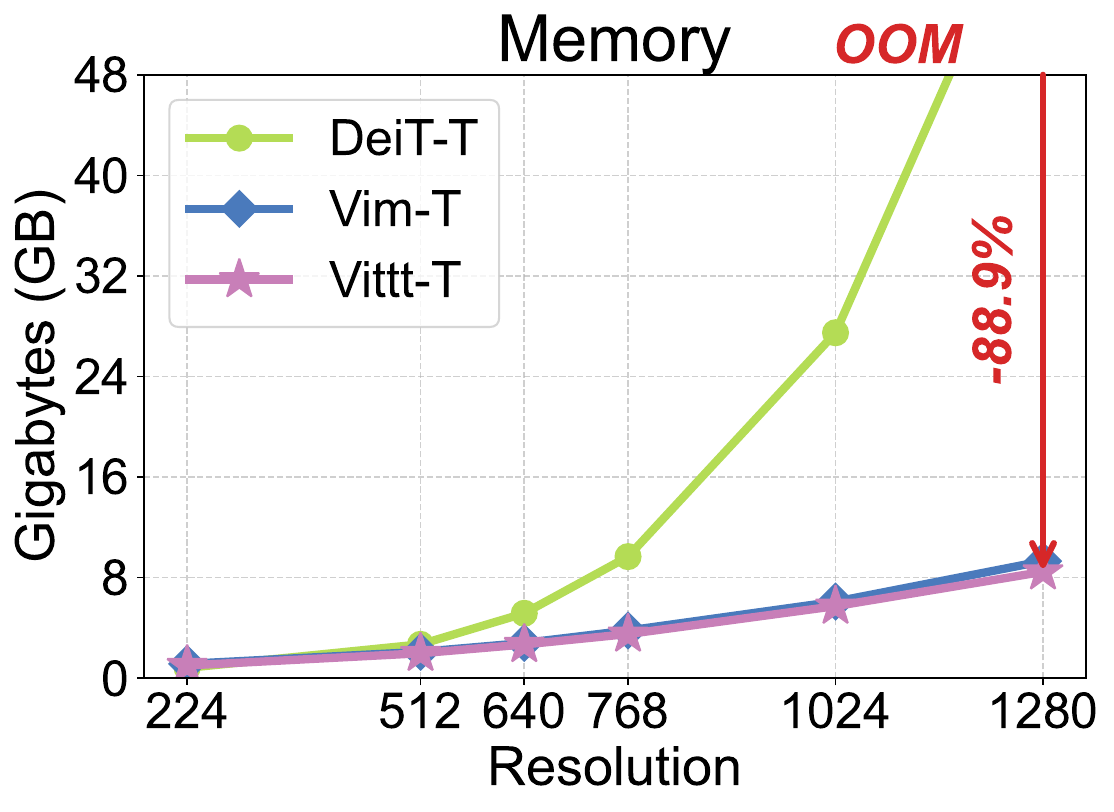}
    \end{subfigure}
    
    \vspace{-0.1in}
    \caption{\small{Performance and efficiency comparison between DeiT~\cite{deit}, Vim~\cite{vim} and our \texttt{Vittt} model. Results show that \texttt{Vittt} not only achieves superior performance on ImageNet classification and downstream detection and segmentation tasks, but also is more computation and memory efficient in dealing with high-resolution images. }}
    \label{fig:comparison}
    \vspace{-0.3in}
\end{figure*}

Despite this, vanilla TTT is originally designed for unidirectional modeling with inherent temporal dependence. This renders its global perception advantages in NLP inadequate for fully global modeling in visual tasks. To introduce spatial locality, we construct dual datasets~\cite{BidirectionalRNN} for the self-supervised learning and incorporate Conv2d-based~\cite{DWConv} dataset preprocessing mechanism. Benefiting from these architectural designs, Vision-TTT exhibits a globally radial Effective Receptive Field (ERF)~\cite{Luo2016UnderstandingTE} to construct representations with explicit 2D correlations. As shown in Fig.~\ref{fig:comparison}, \texttt{Vittt} greatly outperforms the strong baselines of DeiT~\cite{deit} and Vim~\cite{vim} in both classification and downstream tasks, while keeping linear computational and memory complexity for high-resolution images. 

The main contributions of this paper are summarized as follows:
\begin{itemize}
    \item[\ding{71}] We propose Vision-TTT, the first generic visual backbone to leverage Test-Time Training RNNs with gradient-driven state adaptation to capture visual semantics and build expressive visual representations.

     \item[\ding{71}] By integrating hardware-aware kernel implementation, Vision-TTT alleviates the quadratic complexity bottleneck of Vision Transformers and achieves linear complexity modeling. At $1280\times1280$ resolution, \texttt{Vittt-T} saves $79.4\%$ FLOPs and runs $4.72\times$ faster with $88.9\%$ less memory than DeiT-T.  

     \item[\ding{71}] With a dedicated 2D architectural design, Vision-TTT expands vanilla TTT to visual representation tasks with spatial locality. Extensive experiments show that \texttt{Vittt-T/S/B} achieve $77.7\%,81.8\%,82.7\%$ Top-1 accuracy on ImageNet classification, and significantly outperform their counterparts on downstream COCO detection and ADE20K segmentation tasks.
\end{itemize}

\vspace{-0.2in}
\section{Related Work}
\vspace{-0.1in}
\noindent \textbf{Visual Representation Learning.} Visual representation Learning serves as the foundation for a wide range of computer vision tasks~\cite{ImageNet,COCO,ADE20K}. It has evolved into two mainstream paradigms, convolution-based method and visual sequence modeling. Convolutional Neural Networks (CNNs)~\cite{AlexNet,VGG,DenseNet,ResNet,HRNet,EfficientNetV1, EfficientNetV2,ConvNeXt,RepLKNet,RegNet, MoreConvNets, DeformableConvs} have long been the dominant architecture in this field due to their efficiency in processing spatial vision data. However, CNNs are limited by the static nature of convolution kernels~\cite{Fukushima1980NeocognitronAS}, which restricts their scalability in performance. In recent years, Vision Transformer (ViT)~\cite{dosovitskiy2020vit} has pioneered the paradigm of visual sequence modeling, and DeiT~\cite{deit} further refines ViT with advanced training techniques and distillation. They ignite a series of follow-up works that incorporate hierarchical designs~\cite{t2tViT,DaVit-DualAttention, CSWinTA, TransformerPyramid, PyramidViT, HiViT} or dedicated convolution operations~\cite{Container, CoAtNet-MarryConvAndAttention}. Despite their remarkable performance, ViTs suffer from the quadratic complexity of the self-attention mechanism~\cite{Transformer} when dealing with high-resolution images. 

\noindent \textbf{Linear Complexity Visual Sequence Modeling.} To pursue linear complexity representation, various RNN-like sequence models~\cite{mamba,RWKV,Yang2023GatedLA,sun2024learning} have attracted growing interest. Building upon State Space Models (SSMs)~\cite{S4,mamba, mamba2}, Vim~\cite{vim} and VMamba~\cite{VMamba} focus on employing 2D scanning strategies to gain global perception. Their variants~\cite{MIM,EF-VMamba,Yang2024PlainMambaIN,Huang2024LocalMambaVS} mainly improve the efficiency-expressiveness trade-off with visually-adapted SSMs, while MambaVision~\cite{MambaVisionAH} adopts a hybrid architecture of SSM and SwinTransformer~\cite{SwinTransformer} blocks for local-to-global perception. In addition, Vision-RWKV~\cite{VisionRWKVEA} introduces temporal decay to capture long-range dependency and ViG~\cite{ViG} adopts gated linear attention~\cite{Yang2023GatedLA} to enhance expressiveness. Except for ViT$^3$~\cite{han2025vit3}, the impact of Test-Time Training (TTT)~\cite{sun2024learning} for generic visual representation still remains under-explored.

\noindent \textbf{Test-Time Training.} The early multi-task learning version of TTT~\cite{TTTDistributionShifts} introduces an auxiliary self-supervised task parallel to the main task, aiming to address the distribution shift problem in test images~\cite{TTTVideoStream,TTTImageObjects}. Recently, TTT~\cite{sun2024learning,sun2023learning,behrouz2025atlas} has been revisited in a meta-learning~\cite{LearningToLearn} perspective for adaptive sequence modeling, which involves an inner loop for self-supervised reconstruction and an outer loop for different main tasks. The machanism has been applied to the domain of video generation~\cite{dalal2025one,zhang2025test} and 3D reconstruction~\cite{chen2025ttt3r}. \textbf{Our work focuses on the meta-learning formulation of TTT~\cite{sun2024learning}} as an RNN-like visual sequence modeling approach to learn efficient and expressive visual representations.

\vspace{-0.1in}
\section{Preliminary}
\vspace{-0.3in}
\begin{figure*}
    \centering
\includegraphics[width=0.75\linewidth]{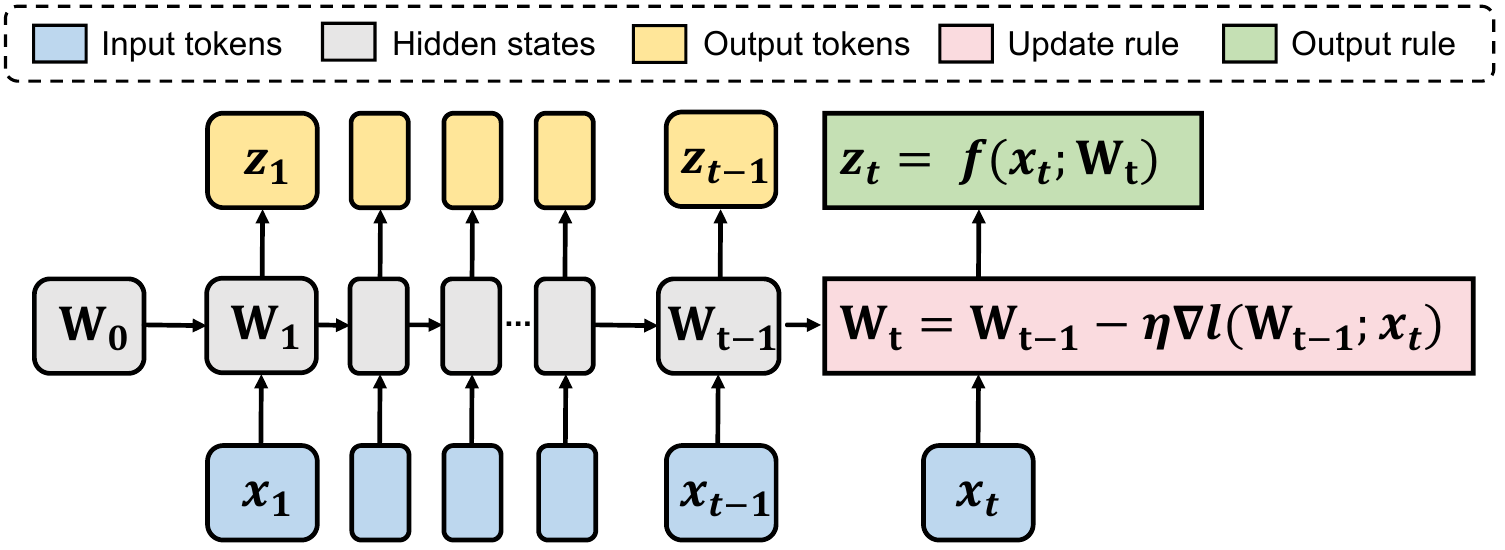}
    \caption{\small{Basics of TTT. The core idea is to update the hidden state $\mathrm{W}$ with steps of self-supervised gradient descent and then forward the output to the next layer.}}
    \label{fig:ttt}
\vspace{-0.2in}
\end{figure*}

Similar to most recurrent models like LSTM~\cite{LSTM}, Mamba~\cite{mamba} and RWKV~\cite{RWKV}, TTT is also a special RNN that maintains a hidden state $\mathrm{W} \in \mathbb{R}^{D \times D}$ and follows an update rule and an output rule. As illustrated in Fig.~\ref{fig:ttt}, the core idea is to treat the token sequence $X=[x_1,\cdots,x_T]$ of a sample as a dataset, where self-supervised learning is performed to compress visual semantics into the hidden state $\mathrm{W}$. \textbf{The update rule} is written as the following gradient descent form,
\begin{equation}
\mathrm{W}_t = \mathrm{W}_{t-1} - \eta \nabla_{\mathrm{W}_{t-1}} \ell(\mathrm{W}_{t-1}; x_t),
\label{eq:ttt_sgd}
\end{equation}
and \textbf{the output rule} for the output sequence $Z=[z_1,\cdots,z_T]$ is given by,
\begin{equation}
z_t = f(x_t;\mathrm{W}_t)=\mathrm{W}_tx_t,
\label{output_rule}
\end{equation}
where $\ell$ is the self-supervised loss, $\eta$ is the learning rate and $f$ represents the mapping from $X$ to $Z$. This adaptation occurs continuously during both training and testing, enabling the model to refine its internal representations in response to observed data patterns, thus called ``Test-Time Training''~\cite{sun2024learning}.

\section{Vision-TTT: Vision Test-Time Training}
\subsection{TTT as a Vision Learner}
\label{sec:ttt_basics}
\noindent \textbf{Formulation of Vision-TTT.} Analogous to the attention mechanism~\cite{Transformer}, given the input visual sequence $X=[x_1,\cdots,x_T]$, we first project it to the key, query and value vectors, 
\begin{equation}
X^K, X^Q,X^V=\theta_KX, \theta_QX, \theta_VX, \label{ttt_kqv}
\end{equation}
where $\theta_K,\theta_Q,\theta_V$ are linear projection matrices.
 The self-supervised task can be formulated as a reconstruction problem of the original input and the target is set to minimize the $L2$ norm loss,
\begin{equation}
\ell(\mathrm{W}_{t-1};x_t)=\Vert{f(x_t^K;\mathrm{W}_{t-1})-x_t^V}\Vert^2=\Vert{\mathrm{W}_{t-1}x_t^K-x_t^V}\Vert^2,  \label{ttt_reconstruct}
\end{equation}
with the output token generated as,
\begin{equation}
z_t=f(x^Q_t;\mathrm{W}_t)=\mathrm{W}_tx_t^Q, \label{ttt_output}
\end{equation}
where $x_t^K,x_t^Q,x_t^V$ are slices of $X^K,X^Q,X^V$ at timestep $t$, respectively. By treating $(X^K,X^V)=\{(x_t^K,x_t^V)\},t=1,...,T$ as the training dataset and $(X^Q)=\{x_t^Q\}$ as the test dataset, this represents reconstructing the $x_t^V$ (label view) with the projected $x_t^K$ (training view) and querying with $x_t^Q$ (test view)~\cite{sun2024learning}. 

\begin{figure*}
    \centering
    \vspace{-0.2in}
\includegraphics[width=0.66\linewidth]{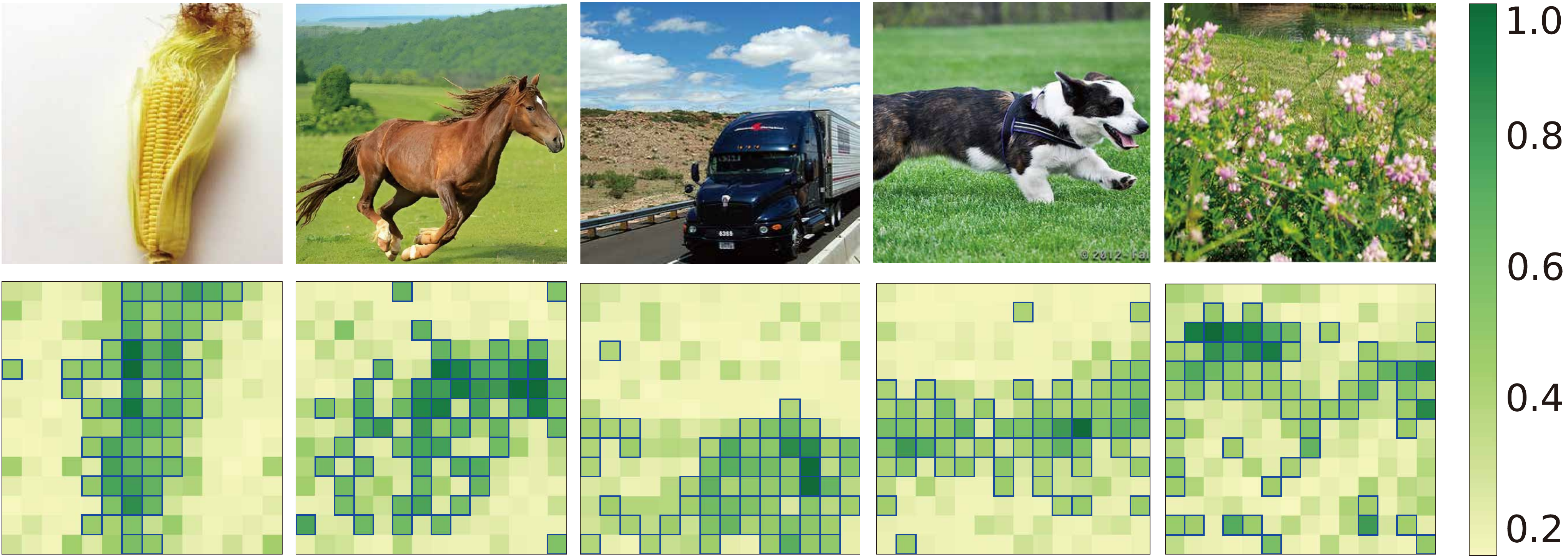}
    \caption{\small{Gradient Magnitude Map (illustrated by Eq.~\eqref{ttt_Gt}). Vision-TTT employs gradients as the explicit indicators to quantify visual token semantics.}}
    \label{fig:GMM}
\vspace{-0.2in}
\end{figure*}

\noindent\textbf{Interpretability of Vision-TTT.} By taking the derivative of $W_{t-1}$ in Eq.~\eqref{ttt_reconstruct} with respect to the self-supervised loss $\ell$, we can obtain the gradient $G_t$ to quantify the token importance at timestep $t$,
\begin{equation}
G_t=\nabla_{\mathrm{W}_{t-1}}\ell(\mathrm{W}_{t-1}; x_t)=2(\mathrm{W}_{t-1}x_t^K-x_t^V)(x_t^K)^T.  \label{ttt_Gt}
\end{equation}
As shown in Fig.~\ref{fig:GMM}, we visualize the $G_t$ distribution on the 2D patch tokens. We observe that tokens with significant visual semantics exhibit larger gradient signal values, which means Vision-TTT has a clear understanding of different regions after training. This provides us with an inherent explainability tool, just as the attention map visualization for vision transformers.

\noindent\textbf{Efficiency of Vision-TTT.} The original token-wise TTT relies on sequential computation, and thus suffers from insufficient efficiency. To fully exploit  the parallelism of modern GPUs, we reduce the hidden size from $\mathrm{W} \in \mathbb{R}^{D \times D}$ to $\mathrm{W} \in \mathbb{R}^{nh\times d \times d}$ with multi-head mechanism~\cite{Transformer}, where $nh$ is the number of heads, $d$ is the head dimension, and $D=nh\times d$. In addition, the granularity of gradient descent is changed from 1 to 16. In other words, we perform the mini-batch gradient descent along the visual token sequence with mini-batch size $b=16$. The update rule within the mini-batch $b$ then follows a causal form~\cite{LinearAttention}. For instance, starting from the initial values of $\mathrm{W_0}$, the hidden state at timestep $t\ (0< t\leq b)$ within the first mini-batch is expressed as,  
\begin{equation}
\mathrm{W}_t=\mathrm{W}_{t-1}-\eta G_t=\mathrm{W}_0-\eta\sum_{s=1}^t G_s.  \label{ttt_minibatch}
\end{equation}
By leveraging the small matrix parallel computing capability of Tensor Cores ($16\times16$), we achieve linear-time throughput. Referring to the great work of TTT in NLP~\cite{sun2024learning}, we encapsulate the forward and backward kernels with Triton~\cite{Tillet2019TritonAI} to transform the linear modeling capability of TTT from theory into reality. Due to space, the detailed kernel implementation is deferred to the Appendix.

\subsection{Overall Architecture}
\label{sec:archi}

\begin{figure*}
    \centering
\includegraphics[width=0.75\linewidth]{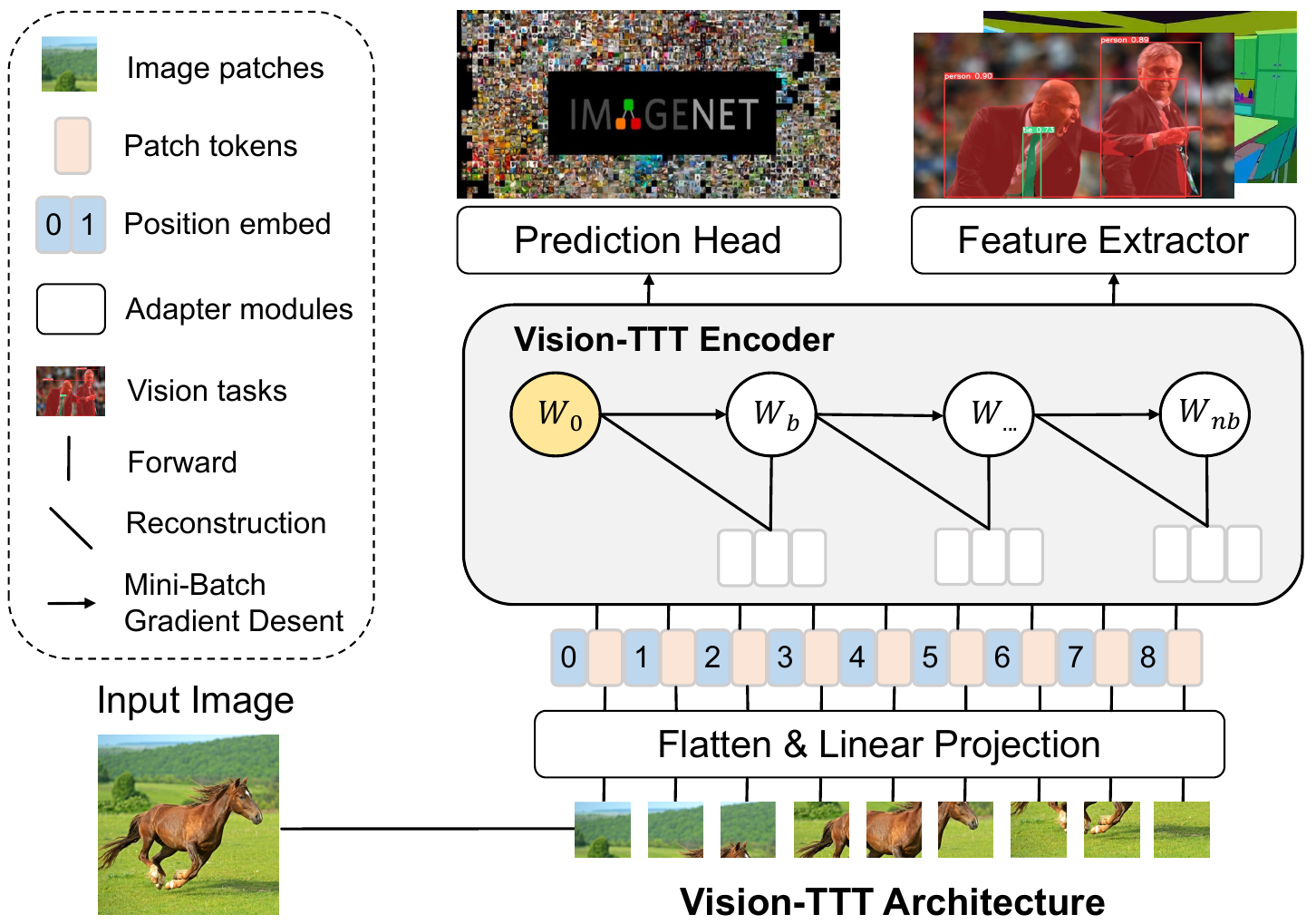}
    \caption{\small{Overall architecture of Vision-TTT. It consists of three stages: Patchification, Vision-TTT Encoder, and Task Adapters to learn representation for vision tasks.}}
    \label{fig:architecture}
\vspace{-0.2in}
\end{figure*}
As illustrated in Fig.~\ref{fig:architecture}, Vision-TTT consists of three stages as follows.

\noindent \ding{202} \textbf{Patchification.} The input image $I\in \mathbb{R}^{H\times W\times C}$ is split into a sequence of 1D patches $x_p\in \mathbb{R}^{T\times(P^2\cdot C)}$~\cite{dosovitskiy2020vit, vim}, where $(H, W)$ is the resolution of the original image, $C$ is the channel, $P^2$ is the size of each rectangle patch and the sequence length is $T=\frac{HW}{P^2}$. Then a learnable projection $E\in \mathbb{R}^{(P^2\cdot C)\times D}$ is appended and a position embedding $E_{\text{pos}}\in \mathbb{R}^{(T+1)\times D}$ is added to form the processed input $y_0$ of feature embedding $D$,
\begin{equation}
y_0 = [x_p^1E, x_p^2E, \cdots, x_p^TE] + E_{\text{pos}},
\end{equation}

\noindent \ding{203} \textbf{Vision-TTT Encoder.} The processed input is then fed into $L$ hybrid blocks. Each of them consists of a \texttt{Vittt} block and a successive \texttt{SwiGluMLP}~\cite{SwiGluMLP}, a variant of Gated Linear Units~\cite{GLU}.
\begin{eqnarray}
\vspace{-0.05in}
&& y_l' = \texttt{Vittt}\bigl(\text{LN}(y_{l-1})\bigr)+y_{l-1}, \nonumber \\
&& y_l = \texttt{SwiGluMLP}\bigl(\text{LN}(y_l')\bigr)+y_{l}',
\vspace{-0.05in}
\end{eqnarray}
where $\text{LN}$ denotes layer normalization, $y_l{'}$ is the intermediate activation. 

\noindent \ding{204} \textbf{Task Adapters.} During the ImageNet~\cite{ImageNet} pretraining stage, we adopt mean pool and a linear detection head to extract the class possibility $\hat{p}$ as follows,
\begin{equation}
\hat{p} = \texttt{Linear}\Bigl(\texttt{MeanPool}\bigl(\text{LN}(y_L)\bigr)\Bigr).
\end{equation}
When finetuning for downstream COCO~\cite{COCO} detection and ADE20K~\cite{ADE20K} segmentation tasks, the common practice is to freeze the encoder backbone and only train the feature extractors for adaptation.

\subsection{Design of \texttt{Vittt} Block}
\label{sec:VitttBlockDesign}
\begin{figure*}
    \centering
\vspace{-0.3in}
\includegraphics[width=1.0\linewidth]{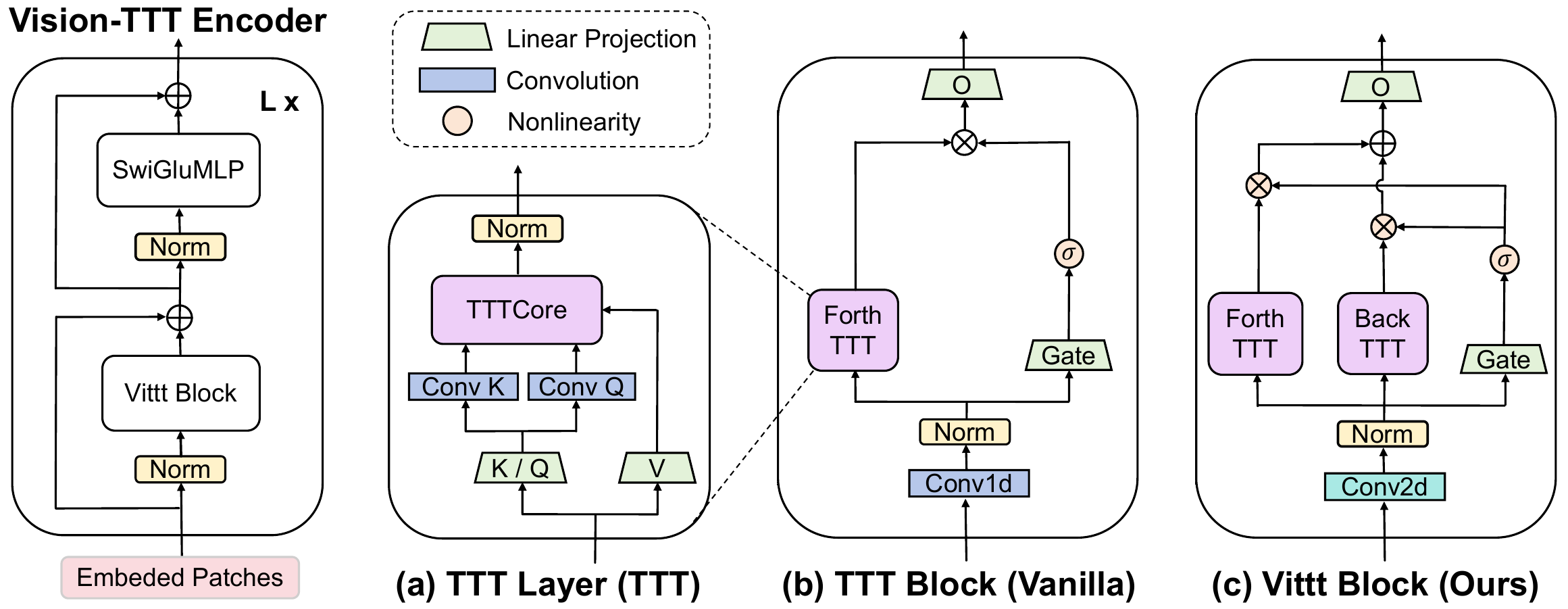}
\vspace{-0.2in}
    \caption{\small{\texttt{Vittt} Block. It evolves from vanilla TTT with the dual dataset strategy and the Conv2d-based dataset preprocessing mechanism for inner self-supervised learning.}}
    \label{fig:vitttblock}
\vspace{-0.2in}
\end{figure*}

The design route from the original TTT layer~\cite{sun2024learning} to our \texttt{Vittt} block is illustrated in Fig.~\ref{fig:vitttblock}. The projection matrices Q / K of the TTT layer are shared to save parameters, and  the vanilla TTT block~\cite{sun2024learning} is essentially a gated linear version. From the view of self-supervised learning in Sec.~\ref{sec:ttt_basics}, the input sequence $X$ gives one dataset $(X^K,X^V)=\{(x_t^K,x_t^V)\},t=1,...,T$. The vanilla TTT (Fig.~\ref{fig:vitttblock} (b)) is formulated with Conv1d-based dataset preprocessing mechanism, 
\vspace{-0.1in}
\begin{equation}
Z_\text{Gated} = \sigma(\texttt{Linear}(X_{\mathrm{Conv1d}})) \odot Z. 
\end{equation}

Although the vanilla TTT block is a promising vision learner, it is originally designed for unidirectional modeling, which conflicts with 2D spatial data. To introduce spatial locality, we further integrate two 2D architectural designs.

\noindent\textbf{Dual Dataset Strategy.} We construct another dataset in the backward direction $(X^{K^{'}},X^{V^{'}})=\{(x_t^{K^{'}},x_t^{V^{'}})\},t=T,...,1$ to alleviate the unidirectional bias of 1D global sequence modeling~\cite{BidirectionalRNN, vim}, where the self-supervised learning performs simultaneously in the forward ($Z_\text{forth}$) and backward ($Z_\text{back}$) routes,
\begin{equation}
Z_\text{Dual} = \sigma(\texttt{Linear} (X_\text{{Conv2d}})) \odot (Z_{\text{forth}} + \texttt{flip}(Z_{\text{back}})). 
\end{equation}
\noindent\textbf{Conv2d-based Dataset Preprocessing.} The $X_\text{Conv2d}=\texttt{DWConv}(X)$ is the pre-processed result of the 1D input sequence $X$ for 2D data augmentation of the self-supervised learning. We use depth-wise convolution~\cite{DWConv} to reduce the parameters introduced ($0.02 \mathrm{M}, 0.04 \mathrm{M},0.08 \mathrm{M}$ for $\texttt{Vittt-T/S/B}$).  

For TTT to perform self-supervised learning along the visual sequence datasets, the dual dataset strategy and the Conv2d-based dataset preprocessing mechanism diversify the datasets to 2D, enabling effective adaptation to visual tasks.  

\vspace{-0.1in}
\section{Experiments}
We conduct extensive experiments to validate the expressiveness of Vision-TTT, including ImageNet~\cite{ImageNet} classification and downstream COCO~\cite{COCO} detection and ADE20K~\cite{ADE20K} segmentation tasks. We also demonstrate its linear efficiency by measuring the FLOPs, throughput and memory metrics across resolutions.


\subsection{Image Classification}
\begin{table}[tb]
  \centering
  \caption{\small{Performance comparison on ImageNet-1K classification. We report the parameters (\#P.), FLOPs and Top-1 Accuracy (\%) metrics. \emph{Left}: hierarchical architectures; \emph{Right}: non-hierarchical architectures.}}
  \label{tab:classification}
  \footnotesize
    \begin{tabular}{l|cc|c|c}
    \toprule
    Method &\text{\ size\ } & \text{\#P.\ }&\text{\ FLOPs\ }&$\text{\ Acc.\ }$ \\
    \midrule
    \multicolumn{5}{c}{\textbf{CNNs}} \\
  \cmidrule{1-5}
    RegNetY-4G & $224^2$& 12M& 4G&80.0\\
    RegNetY-8G & $224^2$ & 25M& 8G&81.7\\
    RegNetY-16G &$224^2$ & 45M& 16G&82.9\\[1pt]
    \midrule
    EffNet-B4  & $380^2$ & 19M& 4.2G& 82.9\\
    EffNet-B5  & $456^2$ & 30M& 9.9G& 83.6\\
    EffNet-B6  & $528^2$ & 43M& 19.0G& 84.0\\[1pt]
    \midrule
    ConvNeXt-T & $224^2$ & 29M& 4.5G& 82.1\\
    ConvNeXt-S  & $224^2$ & 50M& 8.7G& 83.1\\
    ConvNeXt-B  & $224^2$ & 89M& 15.4G& 83.8\\[1pt]
    \midrule
    \multicolumn{5}{c}{\textbf{Transformers}} \\
  \cmidrule{1-5}
    Swin-T  & $224^2$ & 28M& 4.5G& 81.3\\
    Swin-S  & $224^2$ & 50M& 8.7G& 83.2\\
    Swin-B  & $224^2$ & 88M& 15.4G& 83.5\\[1pt]
    \midrule
    \multicolumn{5}{c}{\textbf{SSMs}} \\
  \cmidrule{1-5}
    VMamba-T  & $224^2$ & 30M& 4.9G&82.5\\
    VMamba-S  & $224^2$ & 50M& 8.7G&83.6\\
    VMamba-B  & $224^2$ & 89M& 15.4G&83.9\\[0.5pt]
    \midrule
    MambaVision-T  & $224^2$ & 32M& 4.4G&82.3\\
    MambaVision-S  & $224^2$ & 50M& 7.5G&83.3\\
    MambaVision-B  & $224^2$ & 98M& 15.0G&84.2\\[0.5pt]
    \bottomrule
  \end{tabular}
  \hfill
    \begin{tabular}{l|cc|c|c}
    \toprule
    Method &\text{\ size\ } & \text{\#P.\ }&\text{\ FLOPs\ }&$\text{\ Acc.\ }$ \\
    \midrule
  \multicolumn{5}{c}{\textbf{Transformers}} \\
  \cmidrule{1-5}
    ViT-B/16  & $384^2$& 86M& 55.4G& 77.9\\
    ViT-L/16 & $384^2$ & 307M& 190.7G & 76.5 \\
    \midrule
    DeiT-T   & $224^2$ & 6M& 1.3G& 72.2\\
    DeiT-S & $224^2$  &22M &4.6G&79.8\\
    DeiT-B    & $224^2$ &86M& 17.6G&81.8 \\
    \midrule
    \multicolumn{5}{c}{\textbf{RWKVs}} \\
  \cmidrule{1-5}
    VRWKV-T   & $224^2$ & 6M& 1.2G& 75.1\\
    VRWKV-S & $224^2$  &24M& 4.6G&80.1\\
    VRWKV-B    & $224^2$ &94M&18.2G& 82.0\\
    \midrule
    \multicolumn{5}{c}{\textbf{SSMs}} \\
  \cmidrule{1-5}
    Vim-T   & $224^2$ & 7M& 1.5G&76.1\\
    Vim-S & $224^2$ & 26M & 5.2G&80.3\\
    Vim-B    & $224^2$ & 98M& 18.8G & 81.9\\
    \midrule
    \multicolumn{5}{c}{\textbf{TTTs}} \\
  \cmidrule{1-5}
      ViT$^3$-T   & $224^2$ & 6M& 1.2G&76.5\\
     ViT$^3$-S & $224^2$ & 24M & 4.8G&81.6\\
     ViT$^3$-B    & $224^2$ & 90M& 18.0G & 82.6\\
    \midrule
    \cellcolor{pink!30}\texttt{Vittt-T} & \cellcolor{pink!30}$224^2$ &\cellcolor{pink!30}7M &\cellcolor{pink!30}1.4G& \cellcolor{pink!30}\textbf{77.7} \\
    \cellcolor{pink!30}\texttt{Vittt-S}  & \cellcolor{pink!30}$224^2$ & \cellcolor{pink!30}26M&\cellcolor{pink!30}5.3G&\cellcolor{pink!30}\textbf{81.8} \\
    \cellcolor{pink!30}\texttt{Vittt-B}    & \cellcolor{pink!30}$224^2$ & \cellcolor{pink!30}102M&\cellcolor{pink!30}20.3G&\cellcolor{pink!30}\textbf{82.7} \\
    \bottomrule
  \end{tabular}
  \vspace{-0.2in}
\end{table}

\noindent \textbf{Settings.} We first evaluate \texttt{Vittt} on the ImgeNet-1K dataset~\cite{ImageNet} and compare its performance against the benchmarks from ViT~\cite{dosovitskiy2020vit}, DeiT~\cite{deit}, and linear-complexity Vision-RWKV~\cite{VisionRWKVEA}, Vim~\cite{vim}, ViT$^3$~\cite{han2025vit3}. Despite using longer sequence lengths, hierarchical baselines~\cite{SwinTransformer,RegNet,EfficientNetV1,ConvNeXt,VMamba,MambaVisionAH} are still included as references. For fair comparison, we mainly follow the training recipe of DeiT~\cite{deit} and SwinTransformer~\cite{SwinTransformer}. The training details are provided in the Appendix.

\noindent\textbf{Main Results.} As shown in Tab.~\ref{tab:classification}, \texttt{Vittt-T/S/B} achieves $77.7\%,81.8\%,82.7\%$ Top-1 accuracy on ImageNet classification, respectively. Compared with other methods, \texttt{Vittt-T/S} have significant improvements, which are $+1.6\%$ and $+1.5\%$ better than Vim-T/S. Even for the base size model, \texttt{Vittt-B} surpasses its counterparts DeiT-B, VRWKV-B and Vim-B by $+0.7\%\sim +0.9\%$, demonstrating the strong expressiveness of Vision-TTT at scale. Notably, \texttt{Vittt} even surpasses the strongest baseline ViT$^3$. The reason is that ViT$^3$ performs only one batch gradient descent step, but \texttt{Vittt} replaces it with sequence-aware (i.e. $T/b$) gradient descent steps, enabling continuous adaptation along the visual sequences. 


\begin{table}[tb]
  \centering
  \caption{\small{Performance comparison on downstream COCO2017 detection (\emph{Left}) and ADE20K segmentation (\emph{Right}) tasks. We report the parameters (\#Param.), FLOPs and $\mathrm{AP^b}$, $\mathrm{AP^m}$, $\mathrm{mIoU}$ metrics, respectively. FLOPs are calculated with $1333\times800$ resolution for detection and $512\times512$ resolution for segmentation tasks. }}
  \label{tab:downstream}
  \footnotesize
  \vspace{-0.1in}
  \begin{tabular}{l|cc|cc}
    \toprule
    Method & \#Param.&\ FLOPs\ &$\mathrm{\ AP^b\ }$&$\mathrm{\ AP^m\ }$\\
    \midrule
    ViT-T   & 8M & 147.1G& 41.6& 37.9\\
    VRWKV-T  & 8M& 67.9G& 41.7& 38.0\\
    Vim-T   & 9M & 76.7G & 41.8 & 38.2\\
    ViT$^3$-T & 8M& 69.2G & 42.0& 38.3\\
    \cellcolor{pink!30}\texttt{Vittt-T} & \cellcolor{pink!30}9M&\cellcolor{pink!30}74.3G  &\cellcolor{pink!30}\textbf{42.9}&\cellcolor{pink!30}\textbf{38.7}\\
    \midrule
    ViT-S   & 28M & 344.5G& 44.9& 40.1\\
    VRWKV-S & 29M  &189.9G& 44.8&40.2\\
    Vim-S & 32M & 203.6G& 44.9 & 40.2\\
    ViT$^3$-S & 30M& 195.6G& 45.4& 40.7\\
    \cellcolor{pink!30}\texttt{Vittt-S}  & \cellcolor{pink!30}32M & \cellcolor{pink!30}206.3G&\cellcolor{pink!30}\textbf{46.3}&\cellcolor{pink!30}\textbf{41.4} \\
    \midrule
    ViT-B   & 100M & 893.3G& 46.8& 41.8\\
    VRWKV-B    & 107M &599.0G&46.8& 41.7\\
    Vim-B    & 111M & 611.5G& 46.6& 41.6\\
    ViT$^3$-B & 103M& 597.1G& 47.4& 42.1\\
    \cellcolor{pink!30}\texttt{Vittt-B}    & \cellcolor{pink!30}115M & \cellcolor{pink!30}646.7G&\cellcolor{pink!30}\textbf{48.6}&\cellcolor{pink!30}\textbf{43.0}\\
    \bottomrule
  \end{tabular}
  \hfill
  \begin{tabular}{l|cc|c}
    \toprule
    Method & \#Param.&\ FLOPs\ &$\mathrm{\ mIoU\ }$\\
    \midrule
    ViT-T   & 8M & 20.9G& 42.6\\
    VRWKV-T   & 8M & 16.6G& 43.3\\
    Vim-T   & 9M & 18.4G& 43.4\\
    ViT$^3$-T & 8M& 16.7G& 43.6\\
   \cellcolor{pink!30}\texttt{Vittt-T} & \cellcolor{pink!30}9M &\cellcolor{pink!30}17.8G&\cellcolor{pink!30}\textbf{44.3}\\
    \midrule
    ViT-S   & 28M & 54.0G& 46.2\\
    VRWKV-S & 29M & 46.3G& 47.2\\
    Vim-S & 32M & 49.7G& 47.0\\
    ViT$^3$-S & 30M& 47.7G& 47.6\\
    \cellcolor{pink!30}\texttt{Vittt-S} & \cellcolor{pink!30}32M & \cellcolor{pink!30}50.3G&\cellcolor{pink!30}\textbf{48.4}\\
    \midrule
    ViT-B   & 100M & 157.9G& 48.8\\
    VRWKV-B  & 107M & 146.0G&49.2\\
    Vim-B    & 111M & 149.2G& 48.9\\
    ViT$^3$-B & 103M& 145.7G& 49.4\\
    \cellcolor{pink!30}\texttt{Vittt-B} & \cellcolor{pink!30}115M & \cellcolor{pink!30}157.7G&\cellcolor{pink!30}\textbf{50.3}\\
    \bottomrule
  \end{tabular}
  \vspace{-0.20in}
\end{table}

\subsection{Downstream Tasks}
\vspace{-0.10in}
\noindent \textbf{Settings.} We benchmark \texttt{Vittt} model against other popular non-hierarchical visual sequence modeling methods, including ViT~\cite{dosovitskiy2020vit}, Vision-RWKV~\cite{VisionRWKVEA}, Vim~\cite{vim} and ViT$^3$~\cite{han2025vit3}. For object detection and instance segmentation, we evaluate \texttt{Vittt} on the COCO2017 dataset~\cite{COCO}. The procedure mainly follows the Vision-RWKV recipe~\cite{VisionRWKVEA} to integrate the ViT-Adapter~\cite{Chen2022ViTAdapter} and Mask R-CNN~\cite{MaskRCNN} detection head on our plain \texttt{Vittt} models. The image resolution is scaled up to $1333\times800$. In addition, we perform semantic segmentation on the ADE20K dataset~\cite{ADE20K}. For a fair comparison, we also follow prior works Vim~\cite{vim} and Vision-RWKV~\cite{VisionRWKVEA} to use UperNet~\cite{UperNet} as the segmentation head. All images are cropped into $512 \times 512$. Further details are listed in the Appendix.

\noindent\textbf{Main Results.} Tab.~\ref{tab:downstream} demonstrates the remarkable superiority of Vision-TTT on dense prediction tasks. \texttt{Vittt-T} outperforms the popular Vim-T by $+1.1\%$ $\mathrm{AP^b}$, $+0.5\%$ $\mathrm{AP^m}$ and $+0.9\%$ $\mathrm{mIoU}$. The advantage is further enlarged to $+1.4\%$ $\mathrm{AP^b}$, $+1.2\%$ $\mathrm{AP^m}$ and $+1.4\%$ $\mathrm{mIoU}$ for \texttt{Vittt-S} against Vim-S. For the base size model, where Vim-B fails to scale effectively, \texttt{Vittt-B} still maintains consistent performance gains of $+2.0\%$ $\mathrm{AP^b}$, $+1.3\%$ $\mathrm{AP^m}$ and $+1.1\%$ $\mathrm{mIoU}$ over its VRWKV-B counterpart. Although ViT$^3$ achieves comparable performance with \texttt{Vittt} in classification, it exhibits a considerable performance gap in downstream dense prediction tasks due to limited gradient descent steps for adaptation. In contrast, the maintenance of \texttt{Vittt}'s performance superiority from $512\times512$ (1024 image tokens) to $1333\times800$ (4200 image tokens) effectively demonstrates the scalable representation capability of Vision-TTT's adaptive gradient descent steps, particularly in addressing more challenging long sequence modeling tasks.

\subsection{Efficiency Analysis}

\begin{figure*}[ht!] 
   \vspace{-0.3in}
    \centering 
    \begin{subfigure}[b]{0.32\linewidth}
        \centering
        \includegraphics[width=1.0\linewidth]{images/flops.pdf} 
    \end{subfigure}
    \hfill
    \begin{subfigure}[b]{0.32\linewidth}
        \centering
        \includegraphics[width=1.0\linewidth]{images/throughput.pdf} 
    \end{subfigure}
    \hfill
    \begin{subfigure}[b]{0.32\linewidth}
        \centering
        \includegraphics[width=1.0\linewidth]{images/memory.pdf}
    \end{subfigure}
    
    \begin{subfigure}[b]{0.32\linewidth}
        \centering
        \includegraphics[width=1.0\linewidth]{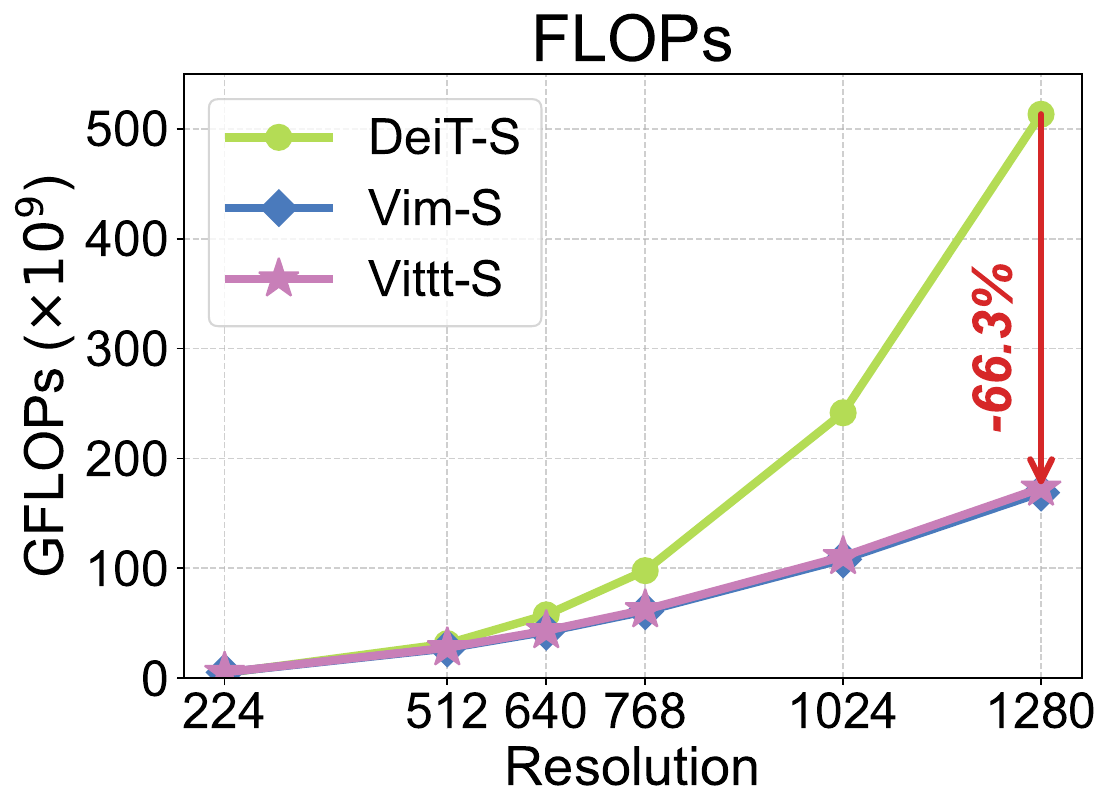} 
    \end{subfigure}
    \hfill
    \begin{subfigure}[b]{0.32\linewidth}
        \centering
        \includegraphics[width=1.0\linewidth]{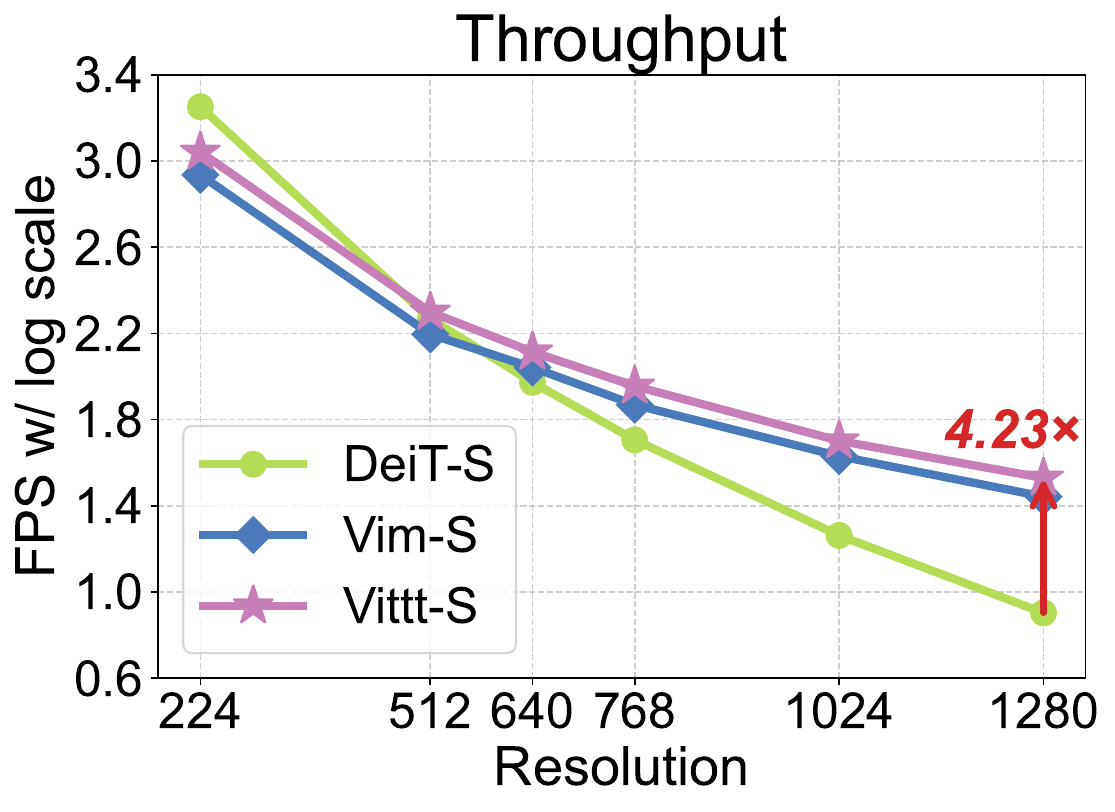} 
    \end{subfigure}
    \hfill
    \begin{subfigure}[b]{0.32\linewidth}
        \centering
        \includegraphics[width=1.0\linewidth]{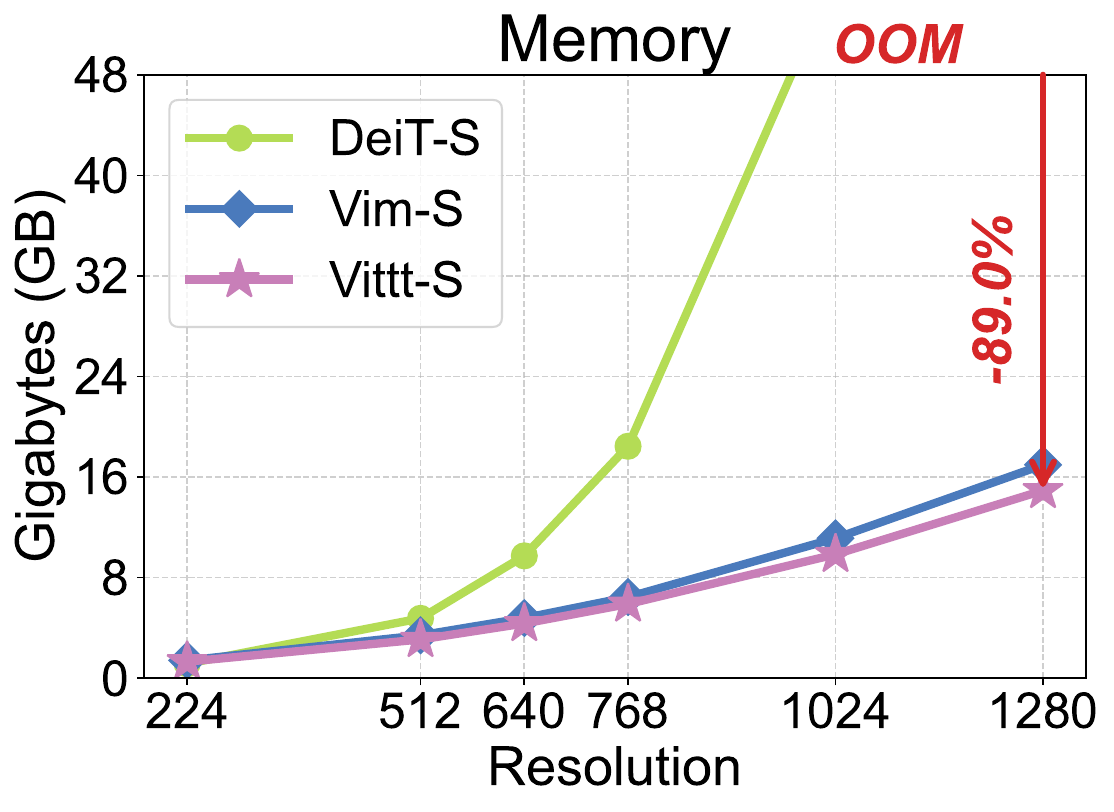}
    \end{subfigure}

        \begin{subfigure}[b]{0.32\linewidth}
        \centering
        \includegraphics[width=1.0\linewidth]{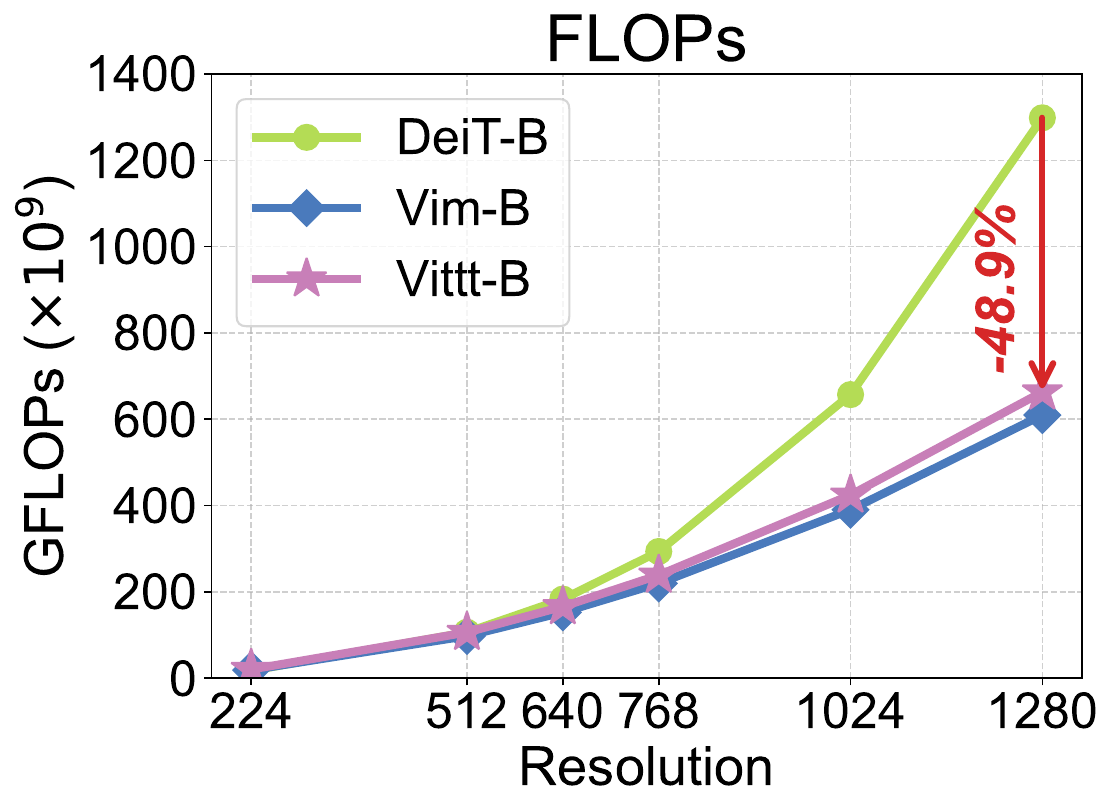} 
    \end{subfigure}
    \hfill
    \begin{subfigure}[b]{0.32\linewidth}
        \centering
        \includegraphics[width=1.0\linewidth]{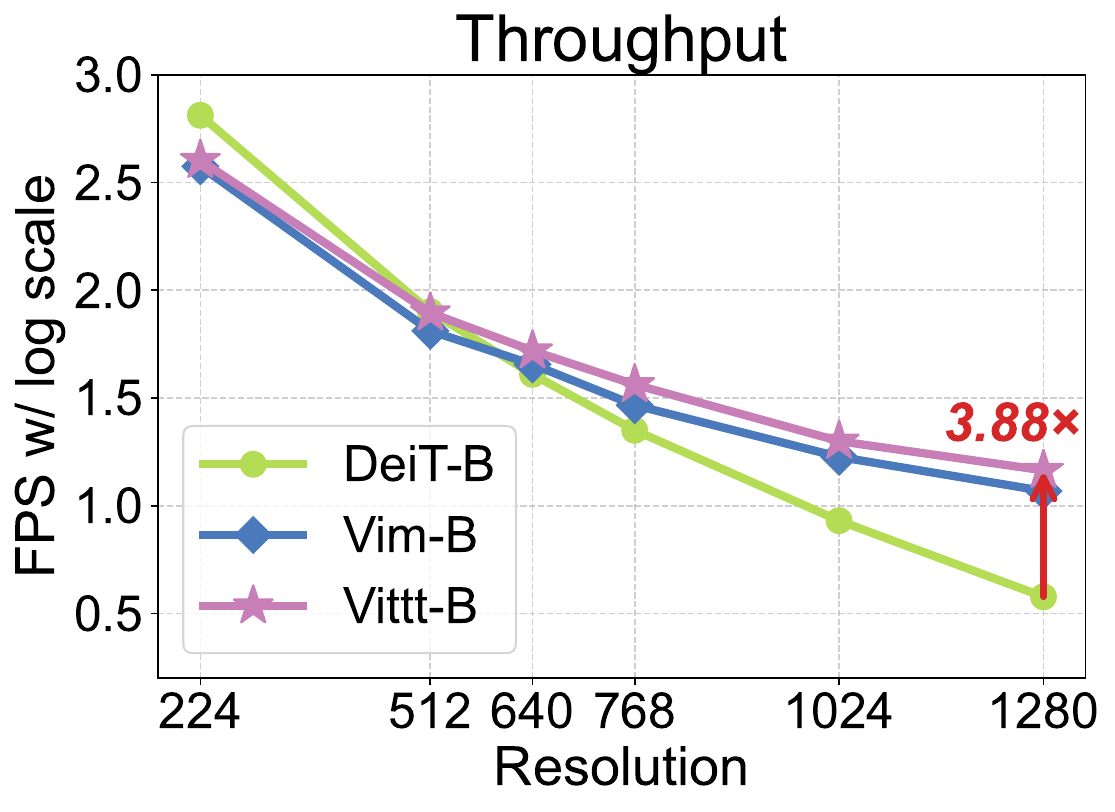} 
    \end{subfigure}
    \hfill
    \begin{subfigure}[b]{0.32\linewidth}
        \centering
        \includegraphics[width=1.0\linewidth]{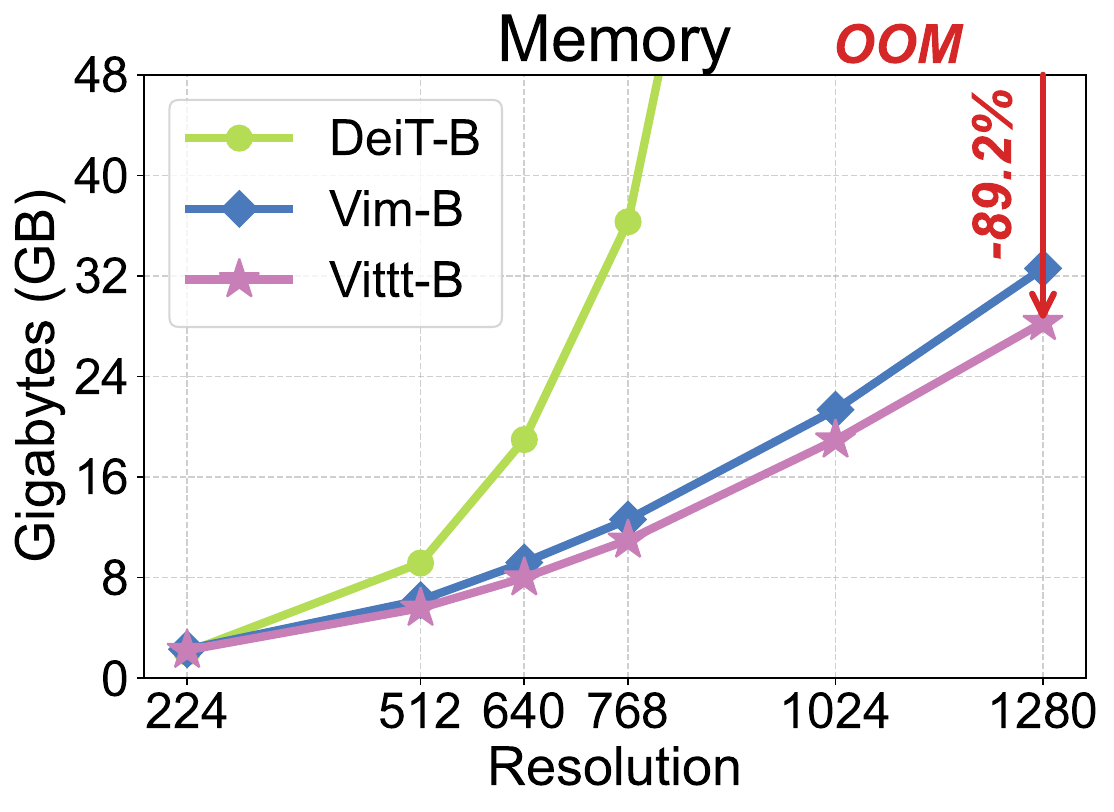}
    \end{subfigure}
    
    \vspace{-0.1in}
    \caption{\small{Efficiency comparison between DeiT~\cite{deit}, Vim~\cite{vim} and our \texttt{Vittt} model. We sketch the FLOPs, throughput, and memory footprint curves for different resolutions ranging from $224\times224$ (i.e. 196 tokens) to $1280\times1280$ (i.e. 6400 tokens). }}
    \label{fig:efficiency}
    \vspace{-0.2in}
\end{figure*}

\noindent\textbf{Settings.} We compare the theoretical complexity and runtime throughput of \texttt{Vittt} with ViT~\cite{dosovitskiy2020vit} and Vim~\cite{vim}. Following~\cite{rw2019timm,SwinTransformer}, we set the batch size to 64 to test the FPS (frames per second) and memory footprint with an L40S GPU.

\noindent \textbf{Computational Efficiency}. Given an input sequence $y\in \mathbb{R}^{1\times T\times D}$, the theoretical computation complexity is given by,
\begin{equation}
\begin{aligned}
& \Omega(\text{ViT})=4TD^2+2T^2D \\
& \Omega(\text{Vim})=6TD^2+18T(2D)N \\
& \Omega(\texttt{Vittt})=6TD^2+6TDd+4bTD,
\end{aligned}
\end{equation}
where $b=16$ denotes the mini-batch size, $d=64$ is the head dimension, and $N=16$ is the state expansion factor of Mamba~\cite{mamba}. By referring to Vim~\cite{vim}, we include the input/output projections, the dicretization steps, the gating operations and then multiply the bidirectional factor to compute the overall complexity of Vim blocks. Obviously, ViT exhibits quadratic complexity w.r.t. the sequence length $T$. In contrast, both Vim and \texttt{Vittt} enjoy linear complexity via RNN-style updates, enabling efficient long-sequence modeling. As shown in the left two columns of Fig.~\ref{fig:efficiency}, the linearly increasing FLOPs and throughput of \texttt{Vittt} align well with the theoretical complexity analysis. Although DeiT achieves slightly higher throughput in processing $224\times224$ images, it slows down quickly with the growth of resolution. At $1280\times1280$ resolution, \texttt{Vittt-T/S/B} saves $79.4\%, 66.3\%, 48.9\%$ FLOPs and achieves $4.72\times, 4.23\times, 3.88\times$ higher FPS than DeiT-T/S/B, respectively. For the linear models \texttt{Vittt} and Vim, \texttt{Vittt-T} has fewer FLOPs than Vim-T, but \texttt{Vittt-S/B} yield slightly more FLOPs than Vim-S/B, implying that \texttt{Vittt} has a larger constant coefficient related to the embedding size $D$. Nevertheless, we observe that \texttt{Vittt} consistently outperforms Vim in speed across different model sizes and input resolutions. This stems from their distinct kernel utilization strategies during execution. While Vim relies on the selective scan mechanism~\cite{mamba} parallelized over CUDA Cores, \texttt{Vittt} directly operates on Tensor Cores, which are $16\times16$ specialized matrix multiplication units that runs several to tens of times faster.

\noindent\textbf{Memory Efficiency}. With the global batch size of $B$, the asymptotic memory complexity is as follows, 

\begin{equation}
\begin{aligned}
&\Omega({\text{ViT}}) = \mathcal{O}(BTD + BT^2) \\
&\Omega({\text{Vim}}) = \mathcal{O}(BTD + BTDN) \\
&\Omega({\texttt{Vittt}}) = \mathcal{O}\left(BTD + BTd +{BTDd}/{b}\right). \\
\end{aligned}
\end{equation}
Linear models of Vim and \texttt{Vittt} alleviate the quadratic memory cost w.r.t. $T$ that ViT demands. Our implementation follows Mamba~\cite{mamba} to incorporate kernel fusion and recomputation techniques, which further reduces the memory footprint to $\mathcal{O}\left(BTD\right)$. As shown in the right column of Fig.~\ref{fig:efficiency}, while the memory consumption of DeiT explodes at high resolutions, the footprint only grows linearly for Vim and \texttt{Vittt}. At $1280\times1280$ resolution, \texttt{Vittt-T/S/B} consume approximately $89.0\%$ less memory than DeiT-T/S/B. We also observe a considerable memory reduction of \texttt{Vittt-S/B} compared to Vim-S/B beyond $1024\times1024$ resolution, indicating that \texttt{Vittt} enjoys a smaller constant factor of memory usage than Vim.

The above experiments consistently demonstrate the linear-complexity computation and memory consumption of Vision-TTT across diverse model scales and sequence lengths. This effectively bypasses the scalability bottleneck of Vision Transformers and renders it highly favorable for high-resolution modeling. For readers with further interest in Vision-TTT's efficiency, we provide the derivation of the computational complexity, the hardware-aware implementation, and additional insights into the memory I/O optimization in the Appendix. 

\vspace{-0.1in}
\subsection{Ablation Study}
\label{sec:ablation}
To verify the effectiveness of our design, we systematically explore the design space to identify favorable configurations. Specifically, we conduct ablation studies on 1) the inner modules of the \texttt{Vittt} block to validate the rationality of our architecture design; 2) different classification strategies to obtain the most compatible supervised learning scheme; 3) mini-batch size $b$ to determine the suitable steps of gradient descent; 4) number of learnable initial states $\mathrm{W_0} \in \mathbb{R}^{nh\times d \times d}$ to enable effective dual dataset learning. The results demonstrate that our design for Vision-TTT is modularly effective, friendly to the inner supervised learning, and achieves a favorable trade-off between model expressiveness and efficiency.

\begin{table}[h]
  \centering
  \caption{\small{Performance comparison for different snapshots on the design route for \texttt{Vittt-T}. Throughput metric FPS (frames per second) is tested with batch size 64.}}
  \label{tab:design_route}
  \footnotesize
    \vspace{-0.1in}
  \begin{tabular}{l|cc|c|c|cc|c}
    \toprule
    Design Route & \#Param.&FLOPs& $\mathrm{\ FPS\ }$ &$\mathrm{\ Acc.\ }$&$\mathrm{\ AP^b\ }$&$\mathrm{AP^m\ }$&$\mathrm{\ mIoU\ }$\\
    \midrule
    TTT layer & 5.87M& 1.14G& 3607 & 74.2 & 40.4 & 36.8& 41.5\\
    + Share Q / K & 5.43M& 1.07G& 3892&74.0& 40.1& 36.5 &41.2\\
    + Gating & 5.89M& 1.15G& 3565&75.2&  41.2& 37.1 &42.0\\
    + Conv1d (vanilla TTT) & 5.90M&1.16G& 3346&75.6& 41.6 & 37.6&42.6\\
    + Dual dataset strategy & 6.96M&1.42G& 2068 &77.5& 42.7 & 38.6 & 44.1\\   
    \cellcolor{pink!30}+ Conv2d dataset preprocessing\ \quad & \cellcolor{pink!30}6.98M&\cellcolor{pink!30}1.44G& \cellcolor{pink!30}2029 &\cellcolor{pink!30}77.7& \cellcolor{pink!30}42.9 & \cellcolor{pink!30}38.7 &\cellcolor{pink!30}44.3\\  
    \bottomrule
  \end{tabular}
 \vspace{-0.2in}
\end{table}

\noindent\textbf{Design Route of the Vittt Block.} The design route presented in Tab.~\ref{tab:design_route} follows the illustration in Sec.~\ref{sec:VitttBlockDesign}. Starting from the TTT layer, the shared Q / K projection first achieves a parameter reduction of $0.44\text{M}$ with no discernible performance degradation. Subsequently, the gating mechanism and Conv1d module are incorporated, endowing TTT with nonlinearity and unidirectional token perception. This integration results in a $14.1\%$ decrease in model throughput, while yielding preliminary gains of $+1.6\%$ in classification accuracy, $+1.5\%$ $\mathrm{AP^b}$, $+1.1\%$ $\mathrm{AP^m}$ for detection and $+1.4\%$ $\mathrm{mIoU}$ for segmentation tasks. To construct the proposed \texttt{Vittt} block, we diversify the visual sequence dataset in two directions. This dual dataset strategy substantially enhances performance by $+1.9\%$ in classification, $+1.1\%$ $\mathrm{AP^b}$, $+1.0\%$ $\mathrm{AP^m}$ in detection and $+1.5\%$ $\mathrm{mIoU}$ in segmentation. The significant improvement justifies the sacrifice of $38.2\%$ in FPS. Finally, the Conv2d dataset preprocessing mechanism further contributes incremental performance gains of $+0.1\%\sim0.2\%$ for different metrics.

\vspace{-0.2in}
\begin{table}[h]
  \centering
  \caption{\small{Performance comparison for different classification strategies on \texttt{Vittt-T}.}}
  \label{tab:classification_strategy}
  \footnotesize
  \vspace{-0.1in}
  \begin{tabular}{l|cc|c|cc|c}
    \toprule
    Method & \#Param.&FLOPs&$\mathrm{\ Acc.\ }$&$\mathrm{\ AP^b\ }$&$\mathrm{\ AP^m\ }$&$\mathrm{\ mIoU\ }$\\
    \midrule
    HeadClassTok & 6.98M& 1.44G& 76.6&  41.6&  37.8& 43.0\\
    MidClassTok & 6.98M& 1.44G& 76.8&  41.7& 37.9 &43.0\\
    DoubleClassTok\quad\quad& 6.98M&1.45G& 75.4& 41.2 &37.2 &42.5\\
    MaxPool& 6.98M&1.44G& 77.2& 42.2 & 38.4 &43.8\\    
    \cellcolor{pink!30}MeanPool& \cellcolor{pink!30}6.98M& \cellcolor{pink!30}1.44G& \cellcolor{pink!30}77.7& \cellcolor{pink!30}42.9 & \cellcolor{pink!30}38.7 & \cellcolor{pink!30}44.3\\
    \bottomrule
  \end{tabular}
  \vspace{-0.2in}
\end{table}

\noindent\textbf{Effects of the Classification Strategy.} Following Vim~\cite{vim}, we experiment with five different ways to obtain the classification feature to identify the optimal scheme for class-supervised learning. The head class token~\cite{dosovitskiy2020vit} refers to insert an extra token before the input patch sequence and using only this token after the output layer to compute the class possibilities. The ``MidClassTok'' and ``DoubleClassTok'' are two variants to insert one class token in the middle or two class tokens before and after the input patch sequence~\cite{vim}, respectively. ``MaxPool'' and ``MeanPool'' are two pooling strategies that extract the $L2$-norm maximum or the average of the output feature sequence. As shown in Tab.~\ref{tab:classification_strategy}, the two pooling strategies achieve significantly better performance than the other three class token strategies, and the ``MeanPool'' strategy is the most suitable for the gradient-driven representation of Vision-TTT.

\begin{figure*}[ht!] 
    \centering 
    \begin{subfigure}[b]{0.49\linewidth}
        \centering
        \includegraphics[width=1.0\linewidth]{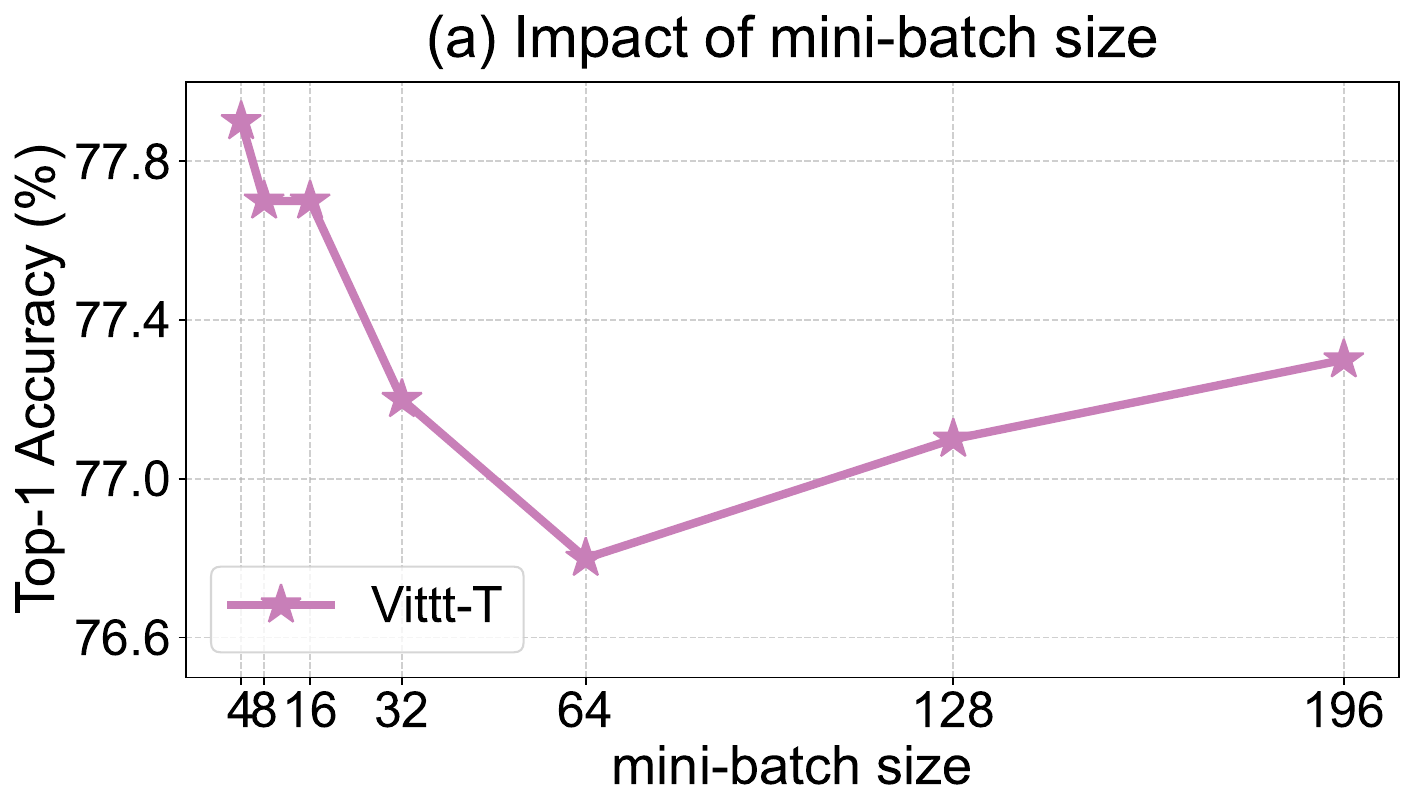} 
    \end{subfigure}
    \begin{subfigure}[b]{0.49\linewidth}
        \centering
        \includegraphics[width=1.0\linewidth]{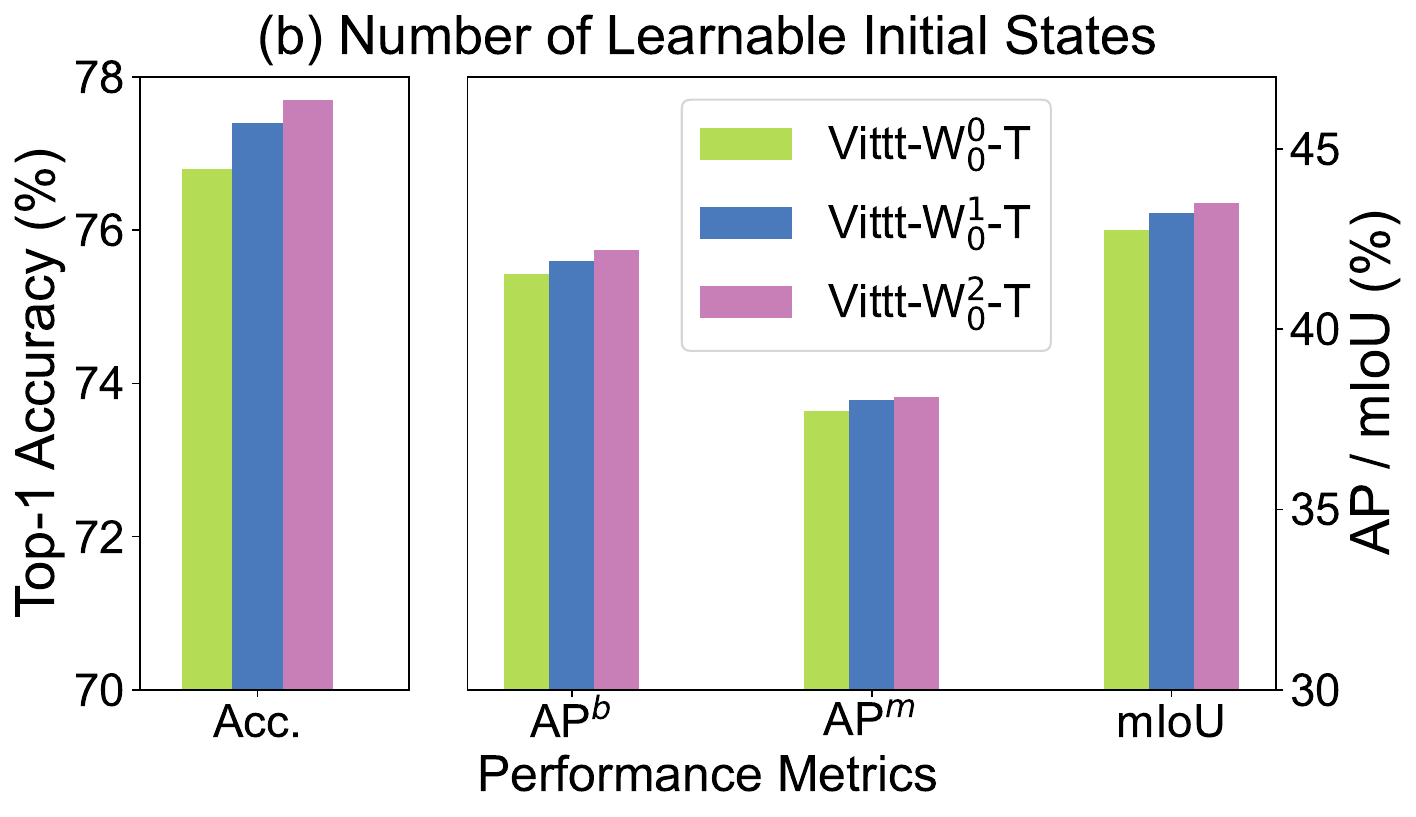} 
    \end{subfigure}
    \vspace{-0.1in}
    \caption{\small{Ablation study of (a) mini-batch size $b$ to perform gradient descent along the sequence length $T$, and (b) number of learnable initial states $\mathrm{W}_0$ (indicated by the superscript number $^0$$^1$$^2$) for the self-supervised learning of dual datasets.}}
    \label{fig:ablation}
    \vspace{-0.2in}
\end{figure*}

\noindent\textbf{Impact of the Mini-Batch Size $b$.} Mini-batch size $b$ along the sequence length $T$ is the crucial hyper-parameter to decide the steps of mini-batch gradient descent (i.e. $\lceil T/b \rceil$ steps). As shown in Fig.~\ref{fig:ablation} (a), with the increase of $b$ from 4 (49 steps of mini-batch gradient descent) to 196 (only one batch gradient descent), the Top-1 classification accuracy first decreases sharply and then grows slightly. The worst case appears in the middle of $b=64$, as there are neither enough gradient descent steps nor a sufficiently global gradient perspective. Fortunately, our kernel implementation of mini-batch size $b=16$ lies in the sweet region from $4\sim16$ to balance the expressiveness and achieved efficiency.

\noindent\textbf{Number of Learnable Initial States.} Recall that in Sec.~\ref{sec:VitttBlockDesign} we adopt dual dataset strategy to achieve bidirectional diversity of 2D spatial data. In the process, there are two initial states $\mathrm{W_0} \in \mathbb{R}^{nh\times d \times d}$ in \texttt{Forth TTT} and \texttt{Back TTT} modules. In such a scenario, the initialization of the hidden states is a noteworthy problem~\cite{sun2024learning}. Unlike general parameter initialization, these states can either be initialized once before training, or set as learnable parameters and trained during supervised learning. We denote the one-time initialization before training as \texttt{Vittt-$\rm{\texttt{W}}_0^0$}, where no hidden states are learnable. The setting \texttt{Vittt-$\rm{\texttt{W}}_0^2$} uses two separate learnable initial states for \texttt{Forth TTT} and \texttt{Back TTT} and \texttt{Vittt-$\rm{\texttt{W}}_0^1$} shares one copy of learnable initial state for the dual datasets. As shown in Fig.~\ref{fig:ablation} (b), more learnable initial states generally reduces interference between the self-supervised learning of dual datasets and leads to improved performance. 

\vspace{-0.1in}
\section{Visualization of Interpretability}
\vspace{-0.4in}
\begin{figure*}[ht!]
    \centering
\includegraphics[width=1.0\linewidth]{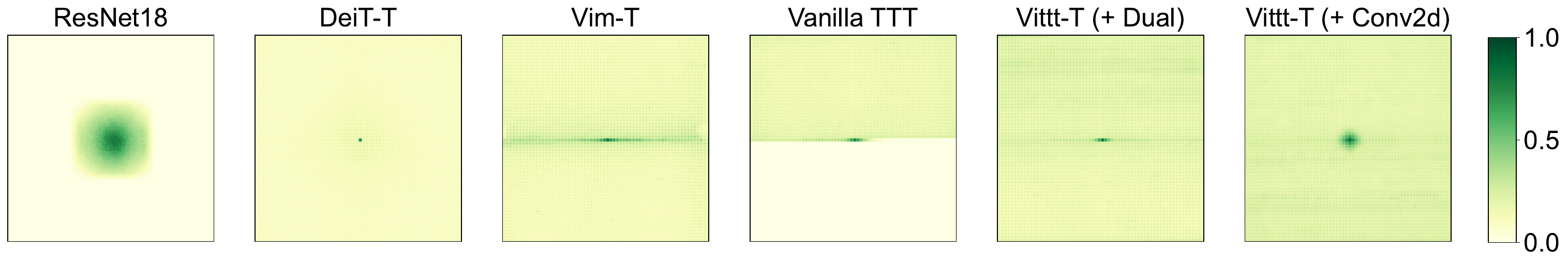}
\vspace{-0.2in}
    \caption{\small{Comparison of Effective Receptive Field (ERF)~\cite{Luo2016UnderstandingTE}. Pixels with higher intensity indicate larger responses related to the central pixel.}}
    \label{fig:erf}
\vspace{-0.2in}
\end{figure*}

\noindent\textbf{Effective Receptive Field (ERF).} The Effective Receptive Field (ERF)~\cite{Luo2016UnderstandingTE,RepLKNet} refers to the region in the input space that contributes to the activations of a specific output unit. As shown in Fig.~\ref{fig:erf}
, we conduct a comparative analysis of the central pixel's ERF across various visual backbones, including ResNet~\cite{ResNet}, DeiT~\cite{deit}, Vim~\cite{vim} and three snapshots of \texttt{Vittt} illustrated in Sec.~\ref{sec:ablation}. From the vanilla TTT to the ``+ Dual'' dataset strategy version, \texttt{Vittt} overcomes the limitations of unidirectional modeling. Further with 2D dataset preprocessing, the ``+ Conv2d'' version ultimately exhibits the globally radial pattern, demonstrating the effectiveness of our design route in Sec.~\ref{sec:VitttBlockDesign}. In addition, we observe that DeiT, Vim, \texttt{Vittt} all achieve global coverage of the pixels, but the 2D-aware distribution of \texttt{Vittt} may intuitively explain its superior expressiveness.

\vspace{-0.2in}
\begin{figure*}[ht!]
    \centering
\includegraphics[width=1.0\linewidth]{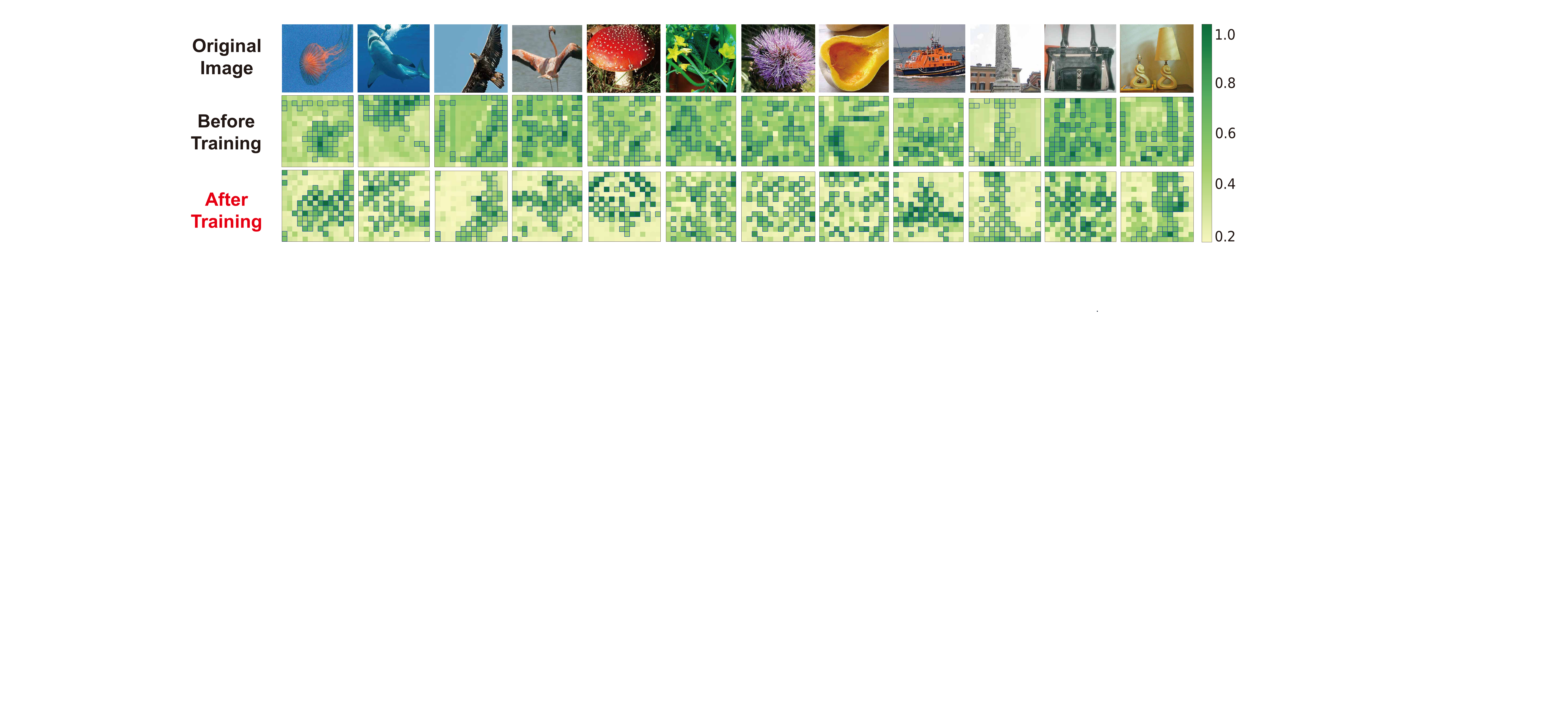}
\vspace{-0.2in}
    \caption{\small{Gradient Magnitude Map of \texttt{Vittt-T} before and after training. Visual tokens with higher intensity indicate steeper gradients (Eq.~\eqref{ttt_Gt})and greater importance.}}
    \label{fig:GMM_2}
\vspace{-0.2in}
\end{figure*}

\noindent\textbf{Gradient Magnitude Map.} Similar to the attention maps~\cite{Transformer} for Vision Transformers and the activation maps~\cite{VMamba} for Vision Mambas, there are also intermediate features to gain an in-depth understanding of Vision-TTT's representation. As shown in Fig.~\ref{fig:GMM_2}, we visualize the Gradient Magnitude Map (GMM) of Vision-TTT before and after training and observe that Vision-TTT develops a discriminative understanding of different input regions. Before training, while the network initialization may partially outline the edges of the main objects, the gradient distribution remains relatively uniform without focusing on specific parts. After supervised learning with target class labels, Vision-TTT prominently highlights the regions with dense visual semantics and diminishes its attention to irrelevant background regions. The GMM serves as a unique interpretive tool for Vision-TTT, offering patch-level explanatory insights that may lay the foundation for future academic advancements of Vision-TTT.

\vspace{-0.1in}
\section{Conclusion}
In this paper, we propose Vision-TTT that introduces a novel sequence modeling method Test-Time Training to the visual representation learning domain. Our method addresses the inherent limitations of Vision Transformers by maintaining a globally radial receptive field meanwhile keeping linear complexity. To adapt the unidirectional TTT model for 2D vision tasks, we propose to utilize the dual dataset strategy for dataset diversity and Conv2d-based module for dataset preprocessing in \texttt{Vittt} blocks. In addition, we integrate kernel implementation for Vision-TTT with linear-complexity parallelism. The significant performance of \texttt{Vittt} across various vision tasks at different model scales, and the verified efficiency for high-resolution images, together demonstrate it as a strong candidate for the next-generation generic visual backbone.


%
%
\bibliographystyle{splncs04}
\bibliography{main}

\clearpage
\begin{center}
  {\LARGE \bfseries Appendix of Vision-TTT\par}
\end{center}
\setcounter{equation}{0}
\appendix
The appendix is intended to elaborate on the implementation details of Vision-TTT and further consolidate the understanding of it as a visual sequence model. Specifically, Appendix~\ref{A} details the kernel implementation; Appendix~\ref{B} describes the algorithm of the \texttt{Vittt} block; Appendix~\ref{C} helps clarify the theoretical modeling capacity of Vision-TTT; Appendix\ref{D} presents the experimental details and Appendix~\ref{E} provides supplementary visual evidence.

\vspace{-0.1in}
\section{Efficient Test-Time Training}\label{A}
\vspace{-0.05in}
Although naive TTT is plagued by serial computation bottleneck, linear wall‑clock time can be achieved with kernel implementation, leveraging the mini‑batch computation pipeline detailed in~\ref{A_1} and the dual-form reformulation in~\ref{A_2}.
\vspace{-0.1in}
\subsection{Mini-batch Gradient Descent}
\label{A_1}
The aforementioned naive TTT layer of updating $W_t$ for every time step $t$ is already efficient in terms of floating point operations (FLOPs) but has insufficient parallelism. Here, we illustrate the special design of the mini-batch gradient descent adopted in~\cite{sun2024learning} to achieve both expressiveness and efficiency. Recall that the general update rule of gradient descent of $W_t$ can be expressed as,
\begin{equation}
W_t=W_{t-1}-\eta G_t=W_0-\eta\sum_{s=1}^tG_s,\nonumber
\end{equation}
where $G_t$ is the descent direction for $W_t$. Once we have calculated $G_t$ for $t=1,...,T$, we can then obtain all the $W_t$ through the \texttt{cumsum} of $\sum_{s=1}^tG_s$. 

\noindent\textbf{Online Gradient Descent.} If we calculate $G_t$ based on the hidden state of the previous step $W_{t-1}$, namely $G_t=\nabla l(W_{t-1};x_t)$, then it is the online gradient descent. Although this is the most expressive setting to utilize full-quantity compression of historical information for reconstruction, it involves intolerable serial computation, leaving most of the computing resources idle. 

\noindent\textbf{Batch Gradient Descent.} If we take all of the gradients w.r.t. $W_0$, then there is only one step of gradient descent for the total sequence. This variant with $G_t=\nabla l(W_0;x_t)$, is known as batch gradient descent. This is the fully-parallelizable version of TTT with actually no explicit adaptation steps, which ends up with relatively poor expressiveness.

\noindent\textbf{Mini-batch Gradient Descent.} Denote the TTT batch size by $b$. We use $G_t=\nabla l(W_{t^{'}};x_t)$, where $t^{'}=t-b$ is the last time step of the previous mini-batch (or 0 for the first mini-batch), so we can parallelize the $b$ gradient computation at a time. Empirically, $b$ controls a trade-off between speed and quantity. Following~\cite{sun2024learning}, we choose $b=16$ by default.

\noindent\textbf{Remarks.} There are two potential channels to propagate information from $W_s$ to $W_t$ for $s<t$: \texttt{cumsum} and the gradient operator. The \texttt{cumsum} is always active, but the gradient operator channel is only active when $W_s$ is from a previous mini-batch. Specifically, the descent step $W_t=W_{t-1}-\eta G_t$ always starts from $W_{t-1}$. Actually, different variants of gradient descent mainly affect the gradient channel, i.e. the descent direction $G_t$, w.r.t. which $W$ the gradient is taken. 

\subsection{Dual Form}
\label{A_2}
\textbf{Background.} Modern accelerators specialize in matrix-matrix multiplications, denoted as \texttt{matmuls}. Without enough \texttt{matmuls} operations, the GPU resource is underutilized. Here, we illustrate how to perform \texttt{matmuls} inside a mini-batch to further improve parallelism. Consider the simplest case of the reconstruction, where $\theta_K=\theta_Q=\theta_V=I$, for only the first mini-batch of size $b$. Consider $f$ as a linear model. The $L2$-norm loss of the self-supervised task at time $t$ is:
\begin{equation}
l(W_0;x_t)=\Vert f(x_t;W_0)-x_t \Vert^2=\Vert W_0x_t-x_t\Vert^2.\nonumber
\end{equation}
Then the natural way is to compute $G_t=\nabla l(W_0;x_t)=2(W_0x_t-x_t)x_t^T$, for $t=1,...,b$. However, we cannot compute all $b$ of the $G_t$s through a single \texttt{matmul} but need $b$ outer products to compute them one by one. Besides, for each $x_t\in R^D$, $G_t$ is $D\times D$, which incurs a larger memory footprint and I/O cost than $x_t$.

\noindent\textbf{Solution.} There is actually no need to materialize $G_1,...,G_b$ as long as we can compute $W_b$ at the end of the mini-batch, and the output tokens $z_1,..,z_b$. Denote $X=[x_1,...,x_b]$, then,
\vspace{-0.2in}
\begin{eqnarray}
W_b&=&W_0-\eta\sum_{t=1}^bG_t \nonumber \\
&=&W_0-2\eta\sum_{t=1}^b(W_0x_t-x_t)x_t^T\nonumber\\ 
&=&W_0-2\eta(W_0X-X)X^T. \
\end{eqnarray}
So $W_b$ can be conveniently computed with a \texttt{matmul}. To compute $Z=[z_1,...,z_b]$,
\begin{eqnarray}
z_t&=&f(x_t;W_t)=W_tx_t\nonumber\\ 
&=&(W_0-\eta\sum_{s=1}^t G_s)x_t\nonumber\\ 
&=&W_0x_t-2\eta\sum_{s=1}^t(W_0x_s-x_s)x_s^Tx_s.\nonumber
\end{eqnarray}
Denote $\delta_t=\sum_{s=1}^t(W_0x_s-x_s)x_s^Tx_s$ and the matrix $\Delta=[\delta_1,...,\delta_b]$. We can derive,
\begin{equation}
\Delta=(W_0X-X)\mathbf{mask}(X^TX),\nonumber
\end{equation}
where $\mathbf{mask}$ is the upper triangle mask with zeros, and the term $W_0X-X$ can be re-used from the computation of $W_b$. Plugging $\Delta$ back, we obtain:
\begin{equation}
Z=W_0X-2\eta\Delta.
\end{equation}
This is the \emph{dual form}, in contrast to the previous \emph{primal form}, where the $G_t$s and $W_s$ are explicitly materialized. As discussed, the two forms are equivalent in the output. For the inherently sequential generation task, the \emph{primal form} already realizes linear-time complexity. Based on the \emph{dual form}, TTT can achieve linear wall-clock time on hardware with pipeline parallelism for training and inference.

\section{Implementation Details of \texttt{Vittt} Block}\label{B}
\subsection{\texttt{Vittt} Block Algorithm}
\vspace{-0.2in}
\begin{algorithm}[!ht]
\caption{\texttt{Vittt} Block}
\label{Algotithm_1}
\begin{algorithmic}[1]
\STATE \textbf{Input}: token sequence $\mathbf{y_{l-1}}: (B, T, D)$          
\STATE \textbf{Output}: token sequence $\mathbf{y_l}: (B, T, D)$       
\STATE \textcolor{gray}{/* Dataset Pre-processing */}
\STATE $\mathbf{y_{l-1}=DWConv(y_{l-1}) + y_{l-1}}$
\STATE $\mathbf{x=Norm(y_{l-1})}$
\STATE \textcolor{gray}{/* Dual-dataset Strategy */}
\STATE $\mathbf{gate=GELU\bigl(Linear^{g}(x)\bigr)}$
\FOR {$\mathbf{d}$ in \{forth, back\}} 
\STATE $\mathbf{x^k =x^q = Linear^{k/q}(x),  x^v = Linear^{v}(x)}$
\STATE $\mathbf{x^k = Conv1d^{k}(x^k), x^q=Conv1d^{q}(x^q)}$
\STATE \textcolor{gray}{/* TTTCore: enumerate the mini-batch along $T$ */}
\FOR {$\mathbf{i}=1$ to$\lceil T/b \rceil+1$}
\STATE $\mathbf{z^{re}_i=W_{i-1}x^k_i}$
\STATE $\mathbf{W_i=W_{i-1}-2\eta(z^{re}_i-x^{v}_i)(x^k_i)^T}$
\STATE $\mathbf{z_i=W_{i-1}x^q_i-2\eta(z^{re}_i-x^v_i)mask\bigl((x^k_i)^Tx^q_i\bigr)}$
\ENDFOR
\ENDFOR
\STATE \textcolor{gray}{/* Gate and merge the output from the two datasets */}
\STATE $\mathbf{z_{forth}^{'}=z_{forth}\bigodot gate}$    
\STATE $\mathbf{z_{back}^{'}=flip(z_{back})\bigodot gate}$  
\STATE $\mathbf{y_{l}=\mathbf{Linear^o}(z_{forth}^{'}+z_{back}^{'})+x}$
\RETURN $\mathbf{y_l}$
\end{algorithmic}
\label{Algorithm_1}
\end{algorithm}
\vspace{-0.2in}
Algorithm \ref{Algorithm_1} illustrates the forward calculation of the \texttt{Vittt} block. Given the input token sequence $\mathbf{y_{l-1}}$ with shape $(B,T,D)$, the goal is to compute the output token sequence $\mathbf{y_l}$ with the same shape. After \texttt{DWConv} and normalization, we obtain the pre-processed input $\mathbf{x}$. Then we have two copies of parameters for \texttt{Forth TTT} and \texttt{Back TTT}. Inside the \texttt{Vittt} layer, we perform linear projection of $\mathbf{x}$ to $\mathbf{x^k,x^q,x^v}$ and apply \texttt{conv1d} to $\mathbf{x^k}$ and $\mathbf{x^q}$ to create separation. The TTTCore reconstructs the token itself with a granularity of mini-batch size $b$. Specifically, it first calculates the forward result $\mathbf{z^{re}_i}$ for reconstruction based on historical encoding $\mathbf{W_{i-1}}$ and current mini-batch $\mathbf{x^k_i}$. Then it updates $\mathbf{W_{i-1}}$ to $\mathbf{W_i}$ with the gradient of the $L2$ norm loss between $\mathbf{z^{re}_i}$ and $\mathbf{x^v_i}$. Finally, the output result $\mathbf{z_i}$ of $\mathbf{W_i}$ and $\mathbf{x^q_i}$ (with the causal masking form) are obtained. All these processes are calculated using matrix multiplication of the \emph{dual form} for efficiency~\cite{sun2024learning}. Outside the TTTCore, there is also a gate and merge operation of the results from the two directions, followed by a linear output projection. 

\textbf{Remarks.} Algorithm~\ref{Algorithm_1} approximately simplifies some details for better illustration. We follow the implementation of~\cite{sun2024learning} to use $\mathbf{f(x)=x+LN(Wx+b)}$ for output rule and multi-head mechanism for $\mathbf{W}$ to reduce the FLOPs. Additionally, the initial hidden state $\mathbf{W_0}$ and learning rate $\mathbf{\eta=\eta_{base}\sigma(Linear^e(x))}$ are both learnable to retain the training stability, where $\mathbf{\sigma}$ is the sigmoid function.  

\subsection{Derivation of the Efficiency Formulae}
Denote $B$ as the batch size, $T$ as the sequence length, and $D$ as the embedding dimension. For \texttt{Vittt}, we adopt a mini-batch size of $b=16$ for hardware efficiency, and the multi-head mechanism for the hidden state $\mathrm{W}\in \mathbb{R}^{nh\times d\times d}$ to reduce FLOPs and memory footprints, where $nh$ is the number of heads, $d=64$ is the head dimension, and $D=nh\times d$. The theoretical computation and memory efficiency formulae are derived as follows.

\noindent\textbf{Computation Efficiency}. Given an input sequence $y\in \mathbb{R}^{1\times T\times D}$, the theoretical computation complexity is given by,
\begin{equation}
\begin{aligned}
& \Omega(\text{ViT})=4TD^2+2T^2D \nonumber\\
& \Omega(\text{Vim})=6TD^2+18T(2D)N \nonumber\\
& \Omega(\texttt{Vittt})=6TD^2+6TDd+4bTD,\nonumber
\end{aligned}
\end{equation}
where $N=16$ is the state expansion factor of Mamba~\cite{mamba}. For ViT~\cite{dosovitskiy2020vit}, $4TD^2$ accounts for the K / Q / V input and output projections and $2T^2D$ is from the calculations associated with the attention map. For Vim~\cite{vim}, $6TD^2$ accounts for the projections of the input, gate, and output with the expanded embedding $E=2D$. The term $18T(2D)N$ is the combination of the $(\Delta,\bar{A},\bar{B}, C)$ discretization steps, the core SSM transformation, and the gating operation multiplied by two for bidirectional scanning paths. For \texttt{Vittt}, by referring to Algorithm~\ref{Algorithm_1}, the first $2TD^2$ is from the linear projection of $\mathbf{x^k,x^q,x^v}$. With mini-batch size $b$, $2bTD$ comes from $T/b$ number of $\mathbf{mask((x_i^k)^Tx_i^q))}$ computation and its multiplication with $\mathbf{(z_i^{re}-x_i^v)}$. Besides, $3TDd$ indicates the total calculation of $\mathbf{W_{i-1}x_i^k,W_{i-1}x_i^q}$ and $\mathbf{(z^{re}_i-x^{v}_i)(x^k_i)^T}$ for all mini-batches in total. Due to the dual dataset strategy, the overall complexity of \texttt{Vittt} is formed by twice the above-mentioned TTT complexity, plus the overhead $2TD^2$ introduced by the gating and output projections. Some other trivial operations inducing less than $\mathrm{o}(bTD)$ calculation have been neglected for complexity analysis.

\noindent\textbf{Memory Efficiency}. Given an input batch $y\in \mathbb{R}^{B\times T\times D}$, the asymptotic memory requirements with batch size $B$ are, 
\begin{eqnarray}
\begin{aligned}
&\Omega(\text{ViT})=\mathcal{O}(BTD+BT^2)\nonumber \\
&\Omega(\text{Vim})=\mathcal{O}(BTD+BTDN)\rightarrow\mathcal{O}(BTD)\nonumber \\
&\Omega(\texttt{Vittt})=\mathcal{O}(BTD+BTd+BTDd/b)\rightarrow\mathcal{O}(BTD). \nonumber
\end{aligned}
\end{eqnarray}

For any sequence model, $\mathcal{O}(BTD)$ memory is the minimal consumption to hold the embedded sequence. For ViT~\cite{dosovitskiy2020vit}, $\mathcal{O}(BT^2)$ accounts for the attention maps. For Vim~\cite{vim}, the major requirement of $\mathcal{O}(BTDN)$ comes from the $T$ intermediate hidden state. For \texttt{Vittt}, as we adopt multi-head mechanism, by referring to Algorithm~\ref{Algorithm_1}, $\mathcal{O}(BbT)$ memory comes from $T/b$ number of $\mathbf{mask((x_i^k)^Tx_i^q)}$ of shape $(b,b)$ and the $\mathcal{O}(BTDd/b)$ memory term comes from $T/b$ number of hidden states each of shape $(nh,d,d)$. Following Mamba~\cite{mamba}, we also incorporate kernel fusion and recomputation techniques to reduce the actual memory I/O to $\mathcal{O}(BTD)$, which decreases the memory consumption by a factor of $d/b=4$.

\section{Theoretical Equivalence}\label{C}
\begin{theorem}[Batch Gradient Descent TTT]\label{thm:linear_attention} Consider the TTT layer with $f(x)=Wx$ as the inner-loop model, batch gradient descent with mini-batch size $b=T$ as the update rule, and $\eta=1/2$, $W_0=0$. Then, given the input sequence $x_1, ..., x_T$, the output rule produces the same output sequence $z_1, ..., z_T$ as linear attention~\cite{LinearAttention}.
\begin{proof} By the definition of the L2-norm reconstruction loss,
\begin{eqnarray}
W_t=\nabla\ell(W_0; x_t)=\nabla \Vert W_{0}\theta_Kx_t-\theta_Vx_t)\Vert^2 =-2(\theta_Vx_t)(\theta_Kx_t)^T. 
\nonumber
\end{eqnarray}
Then with batch gradient descent, 
\begin{eqnarray}
W_t=W_{0}-\eta\nabla\ell(W_0; x_t)=\sum_{s=1}^{t}(\theta_Vx_s)(\theta_Kx_s)^T.
\nonumber
\end{eqnarray}
Plugging $W_t$ into the output rule, we obtain the output token,
\begin{eqnarray}
z_t=f(\theta_Qx_t; W_t)=\sum_{s=1}^{t}(\theta_Vx_s)(\theta_Kx_s)^{T}(\theta_Qx_t),
\label{thm:bgd}
\end{eqnarray}
which is the definition of linear attention~\cite{LinearAttention}.  $\hfill\square$
\end{proof}
\end{theorem}

\begin{theorem}[Online Gradient Descent TTT]\label{thm:more_linear_attention} Consider the TTT layer with $f(x)=Wx$ as the inner-loop model, online gradient descent with mini-batch size $b=1$ as the update rule, and $\eta=1/2$, $W_0=0$. Then, given the input sequence $x_1, ..., x_T$, the output rule produces the output sequence $z_1, ..., z_T$ in an adapted value view form of linear attention.
\begin{eqnarray}
W_t=W_{t-1}-(W_{t-1}\theta_Kx_t-\theta_Vx_t)(\theta_Kx_t)^T, \nonumber
\end{eqnarray}
Plugging $W_t$ into the output rule and for $s=t,...,1$ recursively substituting the left $W_{s-1}$ with $W_s$,
\begin{eqnarray}
\hspace{-0.1in}
&&z_t=f(\theta_Qx_t;W_t)\nonumber \\
&& =[W_{t-1}-(W_{t-1}\theta_Kx_{t}-\theta_Vx_{t})(\theta_Kx_{t})^T](\theta_Qx_t)\nonumber\\
&& =[W_{t-2}-\sum_{s=t-1}^{t}(W_{s-1}\theta_Kx_s-\theta_Vx_s)(\theta_Kx_s)^T](\theta_Qx_t)\nonumber\\
&&=\cdots=W_0(\theta_Qx_t)\nonumber\\
&&+\sum_{s=1}^{t}\underbrace{(\theta_Vx_s-W_{s-1}\theta_Kx_s)}_{\rm{adapted \ value \ view}}(\theta_Kx_s)^T (\theta_Qx_t),
\label{theorem1_eq2} 
\end{eqnarray}
where the value view $\theta_Vx_s$ is adapted with the intermediate state expansion $-W_{s-1}\theta_Kx_s$. This formulation offers the best theoretical expressiveness for $T$ steps of gradient adaptation. $\hfill\square$
\end{theorem}
\textbf{Remarks.} Clearly, the batch gradient descent TTT has limited expressiveness and the online gradient descent TTT suffers from serial computation. While the visual sequence modeling method ViT$^3$~\cite{han2025vit3} adopts batch gradient descent and reduces TTT to a gated linear attention variant, Vision-TTT employs mini-batch gradient descent TTT to construct gated linear attention with $T/b$ ($b=16$) steps of gradient adaptation, maintaining both expressiveness and hardware efficiency.

\section{Experimental Details}\label{D}
\noindent\textbf{Image Classification.} The proposed \texttt{Vittt-T/S/B} are pretrained on ImageNet-1K~\cite{ImageNet} dataset, which contains $1.28$M training images and $50$K validation images from 1000 classes. Following the data augmentation of DeiT~\cite{deit}, we process $224^2$ input images with a total batch size of 1024. The training code is based on SwinTransformer~\cite{SwinTransformer}. Specifically, we employ the AdamW optimizer for 300 epochs with a cosine decay learning rate schedule and 20 epochs of linear warm-up. The initial learning rate is set to 1e-3 and a weight decay of $0.05$ is used. All experiments are performed on $8\times$ L40S GPUs.

\noindent\textbf{Object Detection.} We conduct object detection and instance segmentation experiments on COCO2017~\cite{COCO} dataset, which contains 118K training and 5K validation images. We follow~\cite{VisionRWKVEA,VMamba} to use Mask R-CNN~\cite{MaskRCNN} as our detection head. All models use a $1\times$ training schedule (i.e., 12 epochs) with a total batch size of 16, and AdamW optimizer with an initial learning rate of 1e-4 and weight decay of 0.05. In addition, multi-scale training mechanism is adopted to scale the images by up to $1333\times800$ resolution.

\noindent\textbf{Semantic Segmentation.} We choose the widely-used semantic segmentation dataset ADE20K~\cite{ADE20K}, which contains 20K training and 2K validation images. We follow~\cite{VisionRWKVEA, vim} to use UperNet~\cite{UperNet} as the segmentation head. The AdamW optimizer with an initial learning rate of 6e-5 is employed for the \texttt{Vittt-S/B} models and 12e-5 for the \texttt{Vittt-T} model. All models are trained for 160k iterations with a total batch size of 16 and a weight decay of 0.01. The images are all scaled to $512\times512$ resolution.

\section{Comprehensive Visualizations} \label{E}
To support Vision-TTT's clear understanding of the images, we provide more visualizations of the gradient magnitude map of different classes from ImageNet-1K validation set (Fig.~\ref{fig:appendix_gmm}). Moreover, we also visualize the reconstruction loss across layers (Fig.~\ref{fig:appendix_loss}) to empirically reveal Vision-TTT's inner self-supervised learning dynamics for continuous adaptation across mini-batches. 
\begin{figure*}[h]
    \centering
    \includegraphics[width=1.0\linewidth]{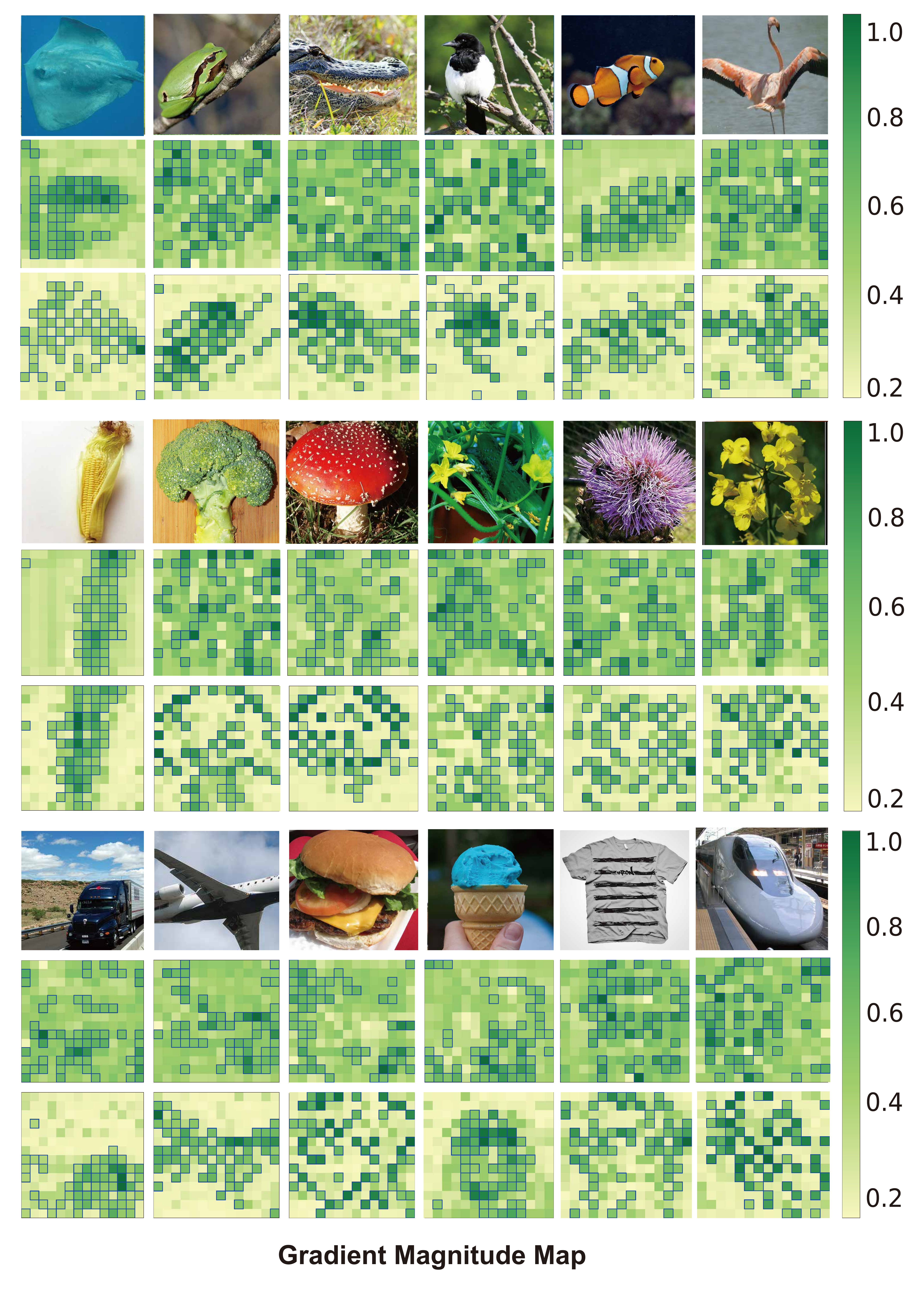}
    \caption{\small{The Gradient Magnitude Map (GMM) of Vision-TTT, which employs gradient indicator $G_t$ to explicitly quantify token importance. Before training (the first row behind the images), the model attributes almost the same gradient magnitude to each token. After training by supervised class label signal (the second row behind the images), the marked Top-30\% tokens clearly show the global semantics of the images.}}
    \label{fig:appendix_gmm}
\vspace{-0.2in}
\end{figure*}
 
\begin{figure*}[h]
    \centering
    \includegraphics[width=1.0\linewidth]{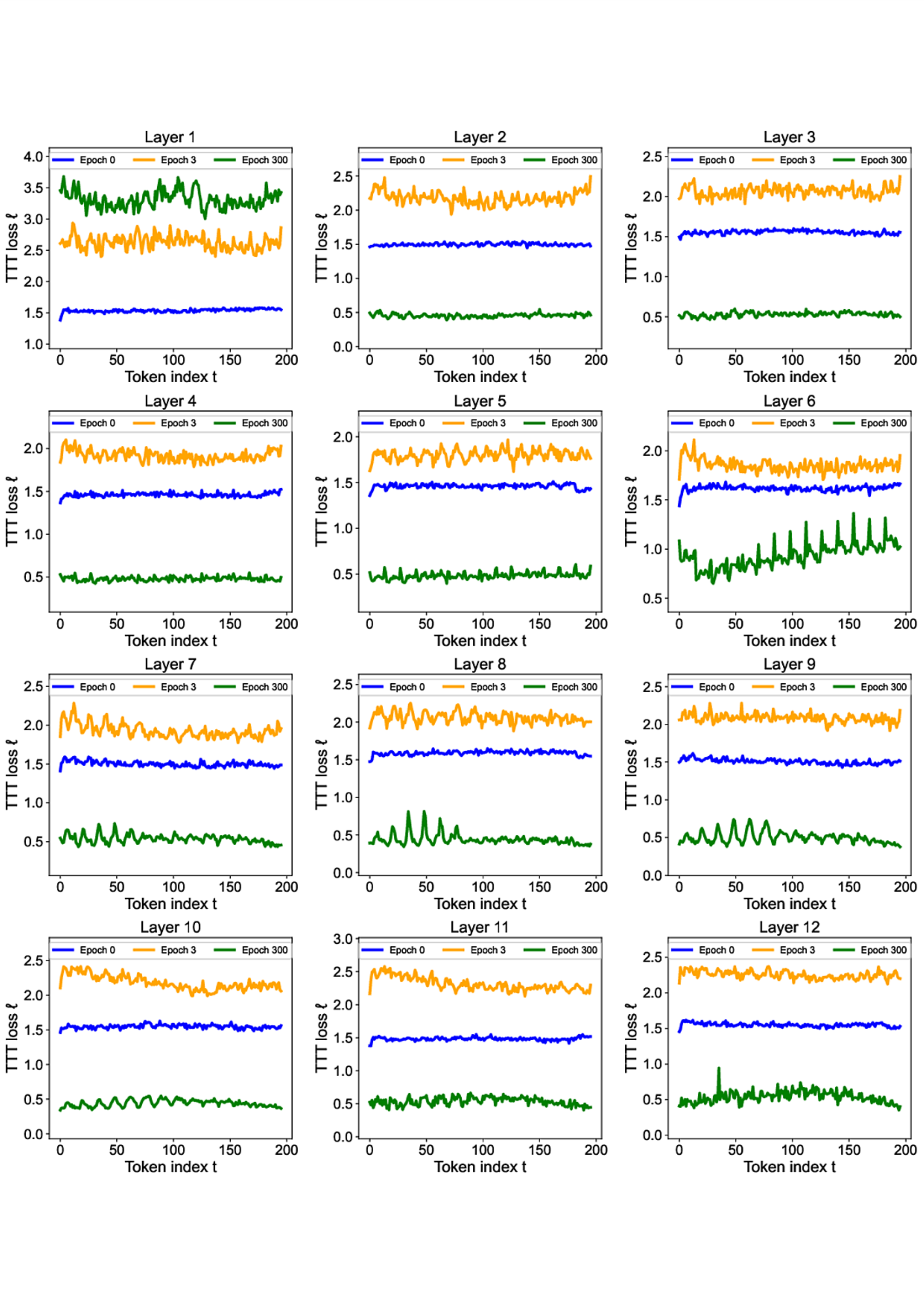}
    \caption{\small{The self-supervised reconstruction loss (L2 norm) of all the 196 tokens for epoch 0, 3, 300 inside 12 layers of \texttt{Vittt-T} (take the \texttt{Forth TTT} as the example), averaged over 128 images. For epoch 0 before training, the reconstruction loss is fairly equal among all the tokens due to random initialization of the network. During training, we observe that the losses initially rise to a certain level before finally dropping to a lower region. This is because at the very beginning the classification signal dominates the optimization to gain more drastic drop of classification loss. As the training progresses, this effect weakens. At a certain stage (around epoch 20-30), the TTT reconstruction loss starts to take over the leadership and performs deeper adaptation. For epoch 300, the final losses among the 12 layers show different patterns for diverse adaptability. Note that they all stay above zero for non-decaying adaptation across mini-batches.}}
    \label{fig:appendix_loss}
\vspace{-0.2in}
\end{figure*}

\end{document}